\documentclass[conference]{IEEEtran}
\IEEEoverridecommandlockouts
\usepackage{cite}
\usepackage{amsmath,amssymb,amsfonts}
\usepackage{algorithm}
\usepackage{algorithmic}
\usepackage{graphicx}
\usepackage{hyperref} 
\usepackage{pifont}
\usepackage{float}
\usepackage{subfig}
\usepackage{textcomp}
\usepackage{amsthm}
 \usepackage{pdfpages}
\usepackage{color}
\usepackage{xcolor}
\def\BibTeX{{\rm B\kern-.05em{\sc 
 i\kern-.025em b}\kern-.08em
    T\kern-.1667em\lower.7ex\hbox{E}\kern-.125emX}}
\begin{document}
\renewcommand{\algorithmicrequire}{\textbf{Input:}}
\renewcommand{\algorithmicensure}{\textbf{Output:}}
\newtheorem{theorem}{Theorem}
\newtheorem{observation}{Observation}
\newtheorem{definition}{Definition}
\definecolor{revise}{rgb}{0,0,0} 
\title{\texttt{Lumos}: Heterogeneity-aware Federated Graph Learning over Decentralized Devices}


\author{\IEEEauthorblockN{Qiying Pan\IEEEauthorrefmark{1}
Yifei Zhu\IEEEauthorrefmark{1}\IEEEauthorrefmark{2}\IEEEauthorrefmark{4}\thanks{\IEEEauthorrefmark{4} Corresponding author}
Lingyang Chu\IEEEauthorrefmark{3}}
\IEEEauthorblockA{\IEEEauthorrefmark{1}UM-SJTU Joint Institute,
Shanghai Jiao Tong University, Shanghai, China}
\IEEEauthorblockA{\IEEEauthorrefmark{2}Cooperative Medianet Innovation Center (CMIC), Shanghai Jiao Tong University, Shanghai, China}
\IEEEauthorblockA{\IEEEauthorrefmark{3}Department of Computing and Software, McMaster University, Hamilton, Canada}
\IEEEauthorblockA{\{sim10\_arity, yifei.zhu\}@sjtu.edu.cn, chul9@mcmaster.ca}
}

\maketitle
\begin{abstract}
Graph neural networks (GNN) have been widely deployed in real-world networked applications and systems due to their capability to handle graph-structured data. However, the growing awareness of data privacy severely challenges the traditional centralized model training paradigm, where a server holds all the graph information. Federated learning is an emerging collaborative computing paradigm that allows model training without data centralization. Existing federated GNN studies mainly focus on systems where clients hold distinctive graphs or sub-graphs. The practical node-level federated situation, where each client is only aware of its direct neighbors, has yet to be studied. In this paper, we propose the first federated GNN framework called \texttt{Lumos} that supports supervised and unsupervised learning with feature and degree protection on node-level federated graphs. We first design a tree constructor to improve the representation capability given the limited structural information. We further present a Monte Carlo Markov Chain-based algorithm to mitigate the workload imbalance caused by degree heterogeneity with theoretically-guaranteed performance. Based on the constructed tree for each client, a decentralized tree-based GNN trainer is proposed to support versatile training. Extensive experiments demonstrate that \texttt{Lumos} outperforms the baseline with significantly higher accuracy and greatly reduced communication cost and training time.
\end{abstract}


\section{Introduction}
Graph-structured data are common in the real world. They can effectively represent complex relations in  social networks, communication networks, and transportation networks. 
Unlike typical data with simple structures, the complex graph topology underlying the graph data sophisticates the corresponding prediction tasks. Graph representation learning addresses the problem by reducing complex graphs into low-dimensional vectors. 
Graph neural networks (GNN) have drawn significant attentions in recent years due to their powerful representativeness. Its outstanding performance relies upon message passing among neighboring vertices and the powerful expressiveness of the neural network. 
GNN has been widely used in networked systems and applications, like gross merchandise value forecast in e-commerce graph\cite{9835292}, item recommendation in knowledge graphs\cite{9835387} and docked bike prediction in spatial-temporal bike-sharing networks\cite{9835338}.

The growing awareness of data privacy and the formal establishment of regulations such as GDPR\cite{EUdataregulations2018} and CCPA\cite{CAdata} motivate a new paradigm called Federated Learning (FL)\cite{mcmahan2017communication}. Instead of training neural network models on a server that stores all the data, FL trains models on local devices without uploading private data to the server.
The server only receives the individually trained models and aggregates them to derive a global one.  Along with efforts to protect typical sequence-like or grid-like data, e.g., texts and images, topological information such as node degree in graph data also calls for stricter privacy protection as they usually embed essential personal information. 

Federated GNN aims to collaboratively train the GNN with privacy preservation of graph data.
Existing federated GNN studies can be categorized into the following three types, depending on how much graph information each device has:
(1) Graph-level separation \cite{xie2021federated,chen2021fedgl} allows each client to have a complete graph with common nodes but different node features and edges. It usually models the collaboration among different organizations with different graph data. (2) Subgraph-level separation \cite{zhang2021subgraph,chen2021fedgraph,du2022federated} splits a graph into small subgraphs and each device stores one subgraph. It is applicable in recommendation user-item networks. (3) Node-level separation \cite{mei2019sgnn,sajadmanesh2021locally} only allows each device to hold its own ego-network, namely, each device is only aware of the connection information between itself and its direct neighbors. It is the most natural form in sensor networks, social networks, and the emerging decentralized applications \cite{decentralizedsocial}. 
Different levels of separation demand specific designs of federated systems. The limited availability of the structure information in a node-separated graph greatly hinders the capacity of the graph learning.
Up till now, 
although existing federated GNN models are capable of learning graph representations while protecting local feature data in node-separated settings, they all leave the critical local node degree  exposed\cite{mei2019sgnn,sajadmanesh2021locally}.


This paper aims to fill the gap to fully protect local feature and node degree in the node-level federated setting. The challenges come from three perspectives. First, \textit{learning}.  Inaccessibility of the complete graph information significantly limits the structural awareness of each device. Obtaining representative node embeddings in such a strict environment calls for a new federated GNN design.
Second, \textit{privacy}. Node features, as well as node degree, all require privacy protection.
How to protect these information from other clients and the central server, preferably being theoretically guaranteed, is non-trivial.
Last but not least, \textit{degree heterogeneity.} Skewed degree distribution in real-world graphs lead to heterogeneous sizes of ego networks across different devices. Directly computing on these ego networks makes certain devices stragglers, and thus prolongs the whole training process. 
How to mitigate the effect of stragglers caused by the degree heterogeneity should be addressed.


This paper introduces \texttt{Lumos}, a federated GNN framework over a decentralized node-separated graph with both features and degrees protected. \texttt{Lumos} comprises a \textcolor{revise}{heterogeneity}-aware tree constructor and a tree-based GNN trainer. The first module, a \textcolor{revise}{heterogeneity}-aware tree constructor, is judiciously designed to construct a tree-structured graph with virtual nodes from the original ego network stored in each device. The vertices\footnote{We use vertex and node interchangeably in this paper.} involved in the tree are wisely selected by a Monte Carlo Markov Chain (MCMC)-based algorithm to balance the computing workload across different devices. 
Each device then considers the constructed tree as a graph to represent GNN layers. 
The tree-based GNN trainer is further proposed to train those trees to learn the representation of vertices in the global graph with privacy preserved. 
A fully decentralized inter-device communication scheme facilitates the aggregation of the embeddings of the same nodes in different devices. Devices then use the aggregated embeddings to compute the loss and improve the neural network. 

In summary, our contributions are as follows:
\begin{itemize}
    \item We propose the first federated GNN framework in the node-separated setting to efficiently support supervised and unsupervised GNN training with degree protected.
    \item A novel tree structure is designed to improve graph representation for each device with only local information.  A tree-based GNN trainer is further proposed with the privacy of feature and degree  protected.
    \item We propose a MCMC-based iterative algorithm to mitigate the workload imbalance problem caused by the newly identified degree heterogeneity. The performance of our algorithm is theoretically guaranteed. 
    \item Extensive experiments demonstrate that \texttt{Lumos} significantly outperforms the federated baseline by a 39.48\% accuracy increase, reducing 35.16\% of inter-devices communication rounds and 17.74\% of training time.
\end{itemize}

The structure of this paper proceeds as follows: We first present a basic introduction to deep graph learning in Section \ref{sec:pr}. Then we review the related works in Section \ref{sec:rw}. In Section \ref{sec:ov}, we provide an overview of the whole system. Section \ref{sec:tc} introduces the first module of \texttt{Lumos}, \textcolor{revise}{heterogeneity}-aware tree constructor; Section \ref{sec:gnn} introduces the other part, tree-based graph neural network training in details. Results based on extensive experiments are presented in Section \ref{sec:ev}, followed by the conclusion  in Section \ref{sec:con}.

\section{Preliminaries on Graph Neural Networks  \label{sec:pr}}
In this section, we present a preliminary introduction to the centralized deep graph learning algorithms. 
GNN has been widely studied in recent years due to its great power in processing graph-structured data \cite{welling2016semi,velivckovic2018graph,lei2019gcn}. It embeds nodes through the aggregation of neighboring embeddings and self-embeddings. The resulting embeddings serve as the representative features to help a wide range of high-level downstream tasks in networked systems, like e-commerce \cite{9835292}, recommendation system\cite{9835387},  transportation networks\cite{9835338}, and user behavior networks\cite{ling2022malgraph}, to name a few.


Graph learning assigns low-dimensional vectors that extract the local features and structural topology to all the vertices. Given a graph $G=(V,E)$ with $\mathbf{X}=(\mathbf{x_v})_{v\in V}$ where $\mathbf{x_v}\in \mathbb{R}^{D}$ refers to the vertex feature of $v$. Graph learning is defined with an encoder and a decoder. The encoder is a mapping from vertices to low-dimensional vectors, 
\begin{equation}
    \text{ENC}:V\rightarrow \mathbb{R}^d.
\end{equation}
In GNN, the encoder is a sequence of message passing layers where each layer is defined as 
\begin{equation}
    \mathbf{h_v^i}=\text{AGG}_i(\mathbf{h_v^{i-1}},\{ \mathbf{h_u^{i-1}}|u\in \mathcal{N}(v)\} ),
    \label{eq:agg}
\end{equation}
where $\mathbf{h_v^i}$ is the embedding output of the $i$th layer, AGG refers to the aggregation function, and $\mathcal{N}(v)$ is the neighbor set for a specific $v$. The initial embeddings $\mathbf{h_v^0}$ are usually set as the local features $\mathbf{x_v}$. Each layer aggregates the embeddings of one vertex and its neighbors so that the generated vector extracts not only local features but also features from neighbors. 
The decoder, denoted as DEC, is a function that reconstructs specific graph data from embeddings. For vertex-related tasks, e.g. node classification, the decoder reconstructs the vertex data from the embeddings.
\begin{equation}
    \text{DEC}(\mathbf{h_v})=\text{READ}(\mathbf{h_v}),
\end{equation}
where READ refers to single or multi-layer perceptrons.
For edge-related tasks, e.g. link prediction, the decoder is usually pairwise, involving edge-related statistics,
\begin{equation}
    \text{DEC}(\mathbf{h_u},\mathbf{h_v})=\text{READ}(\mathbf{h_u}\cdot \mathbf{h_v}).
\end{equation}
 The graph learning problem can be formulated in vertex-related tasks as 
\begin{equation}
    \min_{w_\text{ENC}}\sum_{v\in V} \mathcal{L}(\text{DEC}(\mathbf{h_v}),y_v),
\end{equation}
where $w_\text{ENC}$ are all the parameters in the encoder function, and $\mathcal{L}$ is a loss function to measure the error of the models to achieve the construction goal.
Similarly, in edge-related tasks, the formulation is
\begin{equation}
    \min_{w_\text{ENC}}\sum_{u,v\in V} \mathcal{L}(\text{DEC}(\mathbf{h_u},\mathbf{h_v}),y_{(u,v)})
\end{equation}
where $y$ can be manual labels or a binary number representing the existence of graph elements. 
\section{Related Work\label{sec:rw}}
In this section we review the current works from three aspects: federated learning, graph neural networks, and federated graph learning. 
\subsection{Federated Learning}
The Federated learning paradigm has advanced in the last few years since it successfully tackles the data privacy issue in machine learning algorithms\cite{mcmahan2017communication}. Through the collaboration of neural network training, federated learning trains models without the centralization of raw datasets. However, federated learning suffers performance degradation due to various types of heterogeneity, like data distribution  heterogeneity\cite{zhao2018federated} and resource heterogeneity\cite{nishio2019client}.
To address these types of heterogeneity, a federated learning system usually modifies the default model aggregation protocol or client sampling strategy through heuristic algorithms, reinforcement learning, and clustering algorithms. 
In \cite{li2019convergence}, researchers theoretically prove that learning rate decay can resolve the non-iidness of data. In \cite{zhan2020experience}, researchers propose an experience-driven algorithm based on deep reinforcement learning to solve a mobile computational resource heterogeneity problem in federated learning. In \cite{abad2020hierarchical}, a clustering-based FL framework is proposed to tackle the network communication heterogeneity problem. However, the characteristics of graph data lead to a new type of heterogeneity, like the skewed distribution of node degrees and different levels of graph sparsity, which is difficult to handle using previous approaches. 
\subsection{Graph Neural Networks}
Graph neural networks have emerged as an efficient tool for graph learning. Various neural architectures are introduced to generate node embeddings and fulfill graph-related tasks, such as GCN\cite{welling2016semi}, GAT\cite{velivckovic2018graph}, and GraphSage\cite{hamilton2017inductive}. However, these GNN architectures suffer from limited expressiveness due to the message-passing framework in vanilla GNNs, validated by the incapability to distinguish non-isomorphic graphs\cite{xu2018powerful} and difficulty in approximating combinatorial problems\cite{sato2019approximation}.  

To overcome this shortage, scientists have modified the message-passing framework so that features can be extracted and aggregated in a new manner. One way is to reconstruct the graph structure by adding virtual nodes. In \cite{jin2018junction}, researchers append virtual nodes to the graph to form junction trees to improve the expressive power of GNN. In \cite{talak2021neural}, a new GNN architecture called Neural Tree operates message passing on H-tree constructed from the original graph by linking virtual nodes to actual nodes. 
Similar to these studies, in our problem, the limited expressive power of vanilla GNN architectures makes embedding nodes in small ego networks inaccurate. Therefore, we also modify the ego network to improve representativeness.

\subsection{Federated Graph Learning}
Federated graph learning collaboratively trains the GNN in a privacy-preserving way. True structural information and raw feature information are not allowed to share with other devices and the server.
Separation of graph statistics across different devices greatly impedes the model training and affects the GNN's representativeness.
Various frameworks have been proposed under different levels of separation to tackle the problem.

In graph-level separation, each client owns one or more complete graphs with common vertices but different features or connection information. GCFL in \cite{xie2021federated} clusters local graphs using gradients of GNN networks to improve overall performance. In \cite{chen2021fedgl}, FedGL takes advantage of prediction results and node embeddings to facilitate global graph self-supervision to address varying node label distribution, different graph sparsity, and non-iid node feature problems. 

For subgraph-level separation, each device owns a subgraph of the global graph. FedSage+ in \cite{zhang2021subgraph} utilizes a missing neighbor generator on top of the system to deal with missing links across subgraphs.  In \cite{chen2021fedgraph}, FedGraph uses a novel cross-client convolution operation to address feature data sharing among devices and an intelligent graph sampling algorithm based on deep reinforcement learning to balance training speed and accuracy. Researchers in  \cite{du2022federated} propose a practical federated graph learning system that deals with the trade-off among GCN convergence error, wall-clock runtime, and neighbor sampling interval. 

For node-level separation, each client only owns an ego network from itself. Namely, each client is only aware of its direct neighbors and the edges to its neighbors.
SGNN in \cite{mei2019sgnn} uses one-hot encoding to encrypt features and introduces a new node distance calculation method based on an ordered degree list to process structural information.  
LPGNN in \cite{sajadmanesh2021locally} encrypts local features using a modified one-bit encoder and a robust training framework to increase inference accuracy in the presence of noisy labels. These works focus on protecting local features, but leaves the edge information exposed. A recent work, FedWalk in \cite{10.1145/3534678.3539308}, modifies Deepwalk method in centralized settings to protect edge information. However, the system is limited in representativeness as GNN outperforms DeepWalk. We focus on designing a federated GNN system while protecting both feature and degree. We also reveal a new workload imbalance problem caused by degree heterogeneity in node-separated setting, and propose a theoretically-proved approximation algorithm to solve it. 


\section{\texttt{Lumos} overview\label{sec:ov}}
In this section, we first describe the problem to solve. We then  provide a general overview of \texttt{Lumos} framework. 

\subsection{Problem Description}
For generality, we assume a networked system with $|V|$ end devices and a server.
We can index the vertices with positive integers $1,2,\cdots ,|V|$. Each device is represented by a vertex $v$ in a graph and devices that have a relationship to each other form an edge. Each device further contains local features that are generated by itself. We can represent the underlying system as a graph $G=(V,E)$, where $V$ denotes the total vertex set and $E$ denotes the total edge set. 
Each device holds its ego network denoted as $\mathcal{E}(v)$ containing all the neighboring vertices\footnote{We use device $v$ to represent the device storing $\mathcal{E}(v)$ in the following.}. $\mathcal{E}(v)$ only contains feature $\mathbf{x_v}$ and label $y_v$ without any node information related to other vertices.  
For example, for a decentralized social networking application, each device relates to a social network account. The social relationship in such an application can be represented as edges. The generated texts, images of each client is the local feature.

The system needs to calculate the node embedding for each device with both the local feature and degree protected from other devices and the server. We can formulate the goal of the system as 
\begin{equation}
    \min_{w_{ENC}} \sum_{v\in V}\sum _{u\in \mathcal{N}(v)}\mathcal{L}(\text{DEC}(\mathbf{h_u}),y_u)
\end{equation}

Inter-device communications are allowed and necessary to accomplish the goal of the system. However, all the devices follow the semi-honest adversary model. In other words, they may infer the private information of others from received data without violating the system protocol. \label{privacy_def}\textcolor{revise}{More specifically, the local features should all be preserved with $\epsilon-$ local differential privacy\cite{dwork2006differential} defined formally in Definition \ref{def:dp}. 
\begin{definition}
\label{def:dp}
Let $\epsilon\in \mathbb{R}^+$ and $\mathcal{R}$ be a randomized mechanism. The mechanism $\mathcal{R}$ is said to preserve $\epsilon-$ local differential privacy on feature $\textbf{x}$ if for any feature pair $(\textbf{x}, \textbf{x'})$,
\begin{equation}
    \text{Pr}\left[\mathcal{R}(\textbf{x})=\textbf{y}\right]\leq e^{\epsilon}\text{Pr}\left[\mathcal{R}(\textbf{x'})=\textbf{y}\right]
\end{equation}
\end{definition}
The node degree follows a zero-knowledge protocol\cite{10.1145/3335741.3335750}, a method by which one party (the prover) can prove to another party (the verifier) that something is true without revealing any information apart from the fact that this specific statement is true. 
The zero-knowledge protocol in our federated system concerning node degree is defined in Definition \ref{def:zkp}. 
\begin{definition}
\label{def:zkp}
Let $u$ and $v$ be two devices that need to compare their node degrees to distribute workloads. A protocol is said to protect the privacy of node degrees under zero-knowledge protocol if one device only knows the comparison result but does not know any additional information about the other device's node degree.
\end{definition}
In our scenario, we use zero-knowledge protocol to protect node degree information when balancing device workload based on the comparison of node degrees.
Specifically, we design a zero-knowledge protocol to compare the degrees of the two devices. The result of the comparison is whether the node degree of one device is bigger or smaller than the other device, which is known to both devices. Other than this result, each device will not know any additional information about the other device's node degree.
}

\textcolor{revise}{After the formal description of our learning goal and the privacy requirements, we next present the definition of degree heterogeneity in the node-level federated graph learning system and its implications upon federated GNN model training. \label{heter_def}
}\textcolor{revise}{
\begin{definition}[Degree Heterogeneity]
Degree heterogeneity refers to skewed degree distribution in real-world graphs which leads to heterogeneous sizes of ego networks across different devices. 
\label{ob:dh}
\end{definition}}
\textcolor{revise}{
Given that the ego networks are of varying sizes, the workload distribution across different devices is skewed in the node-level federated graph learning system. Since the devices in practice usually have limited computing power and storage resources, directly computing on these ego networks makes some devices run slower than others, also known as stragglers in distributed systems, and thus prolongs the whole training process. Furthermore, the imbalanced workload caused by degree heterogeneity also introduces high computation costs on certain edge devices. Hence, the system should be wisely designed by reducing the size of large local graphs to avoid excessive communication and computation costs.}
\subsection{\texttt{Lumos}: Solution Overview}
\begin{figure}[!h]
\centering
\includegraphics[width=3.45in]{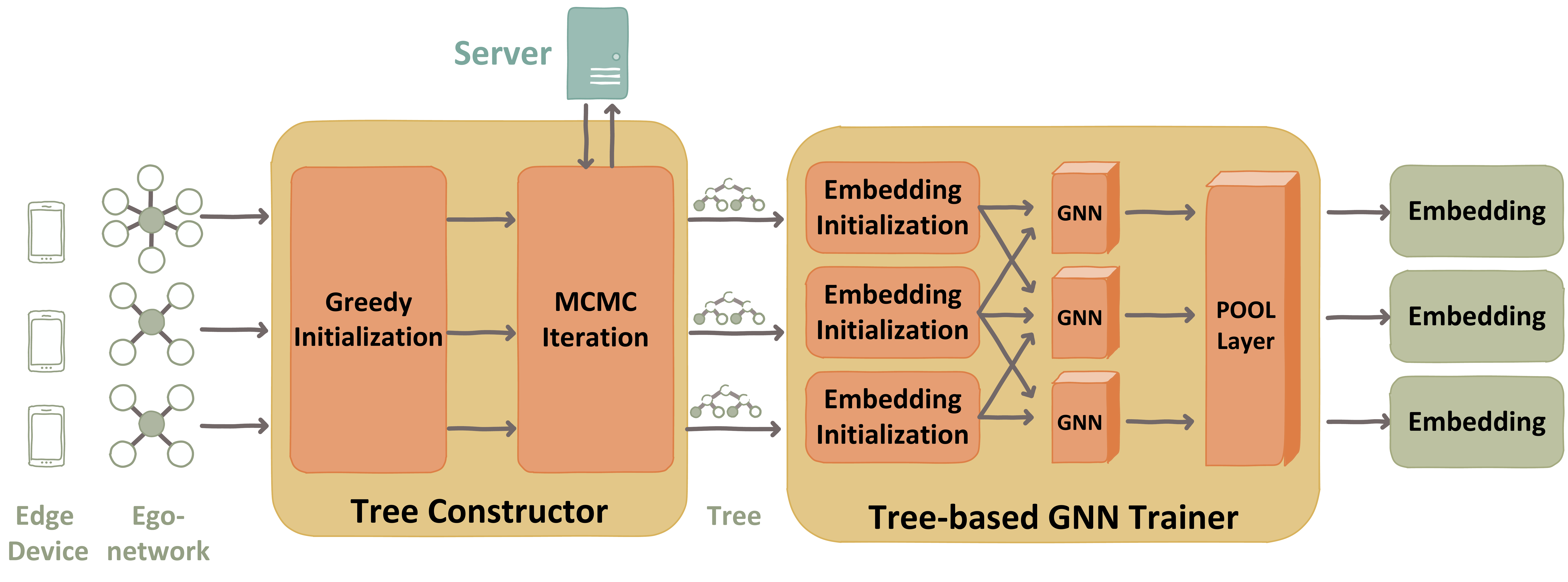}
\caption{\texttt{Lumos} overview: unique blocks (greedy initialization, MCMC iteration and POOL layer) represent parts involving inter-device communication and duplicated blocks (embedding initialization and GNN) represent parts operated locally.}
\label{fig:overview}
\end{figure}

As shown in Fig.\ref{fig:overview}, \texttt{Lumos} comprises two modules, a \textcolor{revise}{heterogeneity}-aware tree constructor and a tree-based GNN trainer. Except the MCMC iteration where server-device communication is required, the rest parts in \texttt{Lumos} are fully decentralized without  server coordination. \texttt{Lumos} is a synchronize federated framework that operates in rounds and has to receive all the required updates to start the next round. 

The original ego network is not expressive enough to train GNN layers as most ego networks contain a minimal number of vertices\cite{clauset2009power}. The tree-constructor converts $\mathcal{E}(v)$ in each device into a particular tree, denoted as $\mathcal{T}(v)$, with leaves corresponding to the vertices in the subgraph and virtual internal nodes. 
\textcolor{revise}{However, the degree heterogeneity introduced in Definition \ref{ob:dh} incurs a heavy workload imbalance problem in the following GNN trainer, slowing the overall completion time for each training epoch.} 
Since each edge exists in two subgraphs, one subgraph can remove that edge to reduce the number of leaves in the tree. The tree constructor uses a greedy initialization and an MCMC iteration to trim trees so that the numbers of leaves are balanced across devices. \textcolor{revise}{During the tree trimming process, the workload values are encrypted to follow the zero-knowledge protocol for comparison. }

A tree-based GNN trainer trains a particular GNN network with the constructed trees. An embedding initialization encrypts features with local differential privacy because of privacy concerns. 
Then the device transmits the features to their neighboring devices. Each device trains GNN layers upon its tree using these encrypted features as initialized embeddings. The layers generate representations for all the leaves associated with vertices in the graph. 
After that, a POOL layer gathers all the embeddings of leaves in multiple trees corresponding to the same vertex in the global graph to generate the vertex embeddings through inter-device communications. 

The system uses multiple approaches \textcolor{revise}{to meet the requirement of privacy}. In the tree-constructor, the exchange of workload statistics implies the degrees of the clients. Hence, the tree-constructor deploys an integer comparison encryption protocol so that the information is exchanged under security. As mentioned, embedding encryption prevents node feature leakage in the graph neural network trainer. Moreover, the system guarantees that labels of nodes are used locally without joining the inter-device communication. 

\section{\textcolor{revise}{Heterogeneity}-Aware Tree Constructor\label{sec:tc}}
In this section, we introduce the first part of \texttt{Lumos}, heterogeneity-aware tree constructor. We first present the intuitive idea of tree construction and identify the potential workload imbalance issue due to degree heterogeneity in the system. Then we formulate the workload balance problem. Last, we propose a tree-trimming solution to the problem. 
\subsection{Initial Tree Construction}
The idea of the tree construction is to insert some virtual nodes into the ego network, forming a particular tree to strengthen its expressiveness. Thus, for device $v$, its tree $\mathcal{T}(v)$ is combined with real vertices in the ego network and virtual nodes. The leaves in $\mathcal{T}(v)$ correspond to the real vertices in $\mathcal{E}(v)$ while the internal nodes are virtual without a direct relationship to the real vertices. The tree is constructed from the bottom to the top. We first create leaf pairs. A leaf pair $(v,u)$ is consisted of the vertex $v$ itself and one of its selected neighbor $u\in \mathcal{N}(v)$. 
\textcolor{revise}{\label{vr}Here, node $v$ is replicated  $|\mathcal{N}(v)|$ times so that the only non-noised feature in this ego-graph\footnote{In subsection \ref{encoder}, it will be introduced why and how other features are noised due to privacy restriction.} is utilized more in later training. Then we add a common parent node to the two leaves in each pair to join them. Last, we use a root node to make all the parent nodes children of the root node in the tree. This tree architecture includes virtual nodes representing subgraphs in the original graph. The virtual root node represents the whole ego network, and each virtual parent node of leaves refers to a two-vertex subgraph, including its two child vertices and the edge between them. These virtual nodes allow the tree to express more structural information about the original ego network. A general version of this tree architecture can learn any graph compatible function and have great expressive power \cite{talak2021neural}. }

\begin{figure}[!h]
\centering
\subfloat[Ego network of vertex 1]{\includegraphics[width=1.2in]{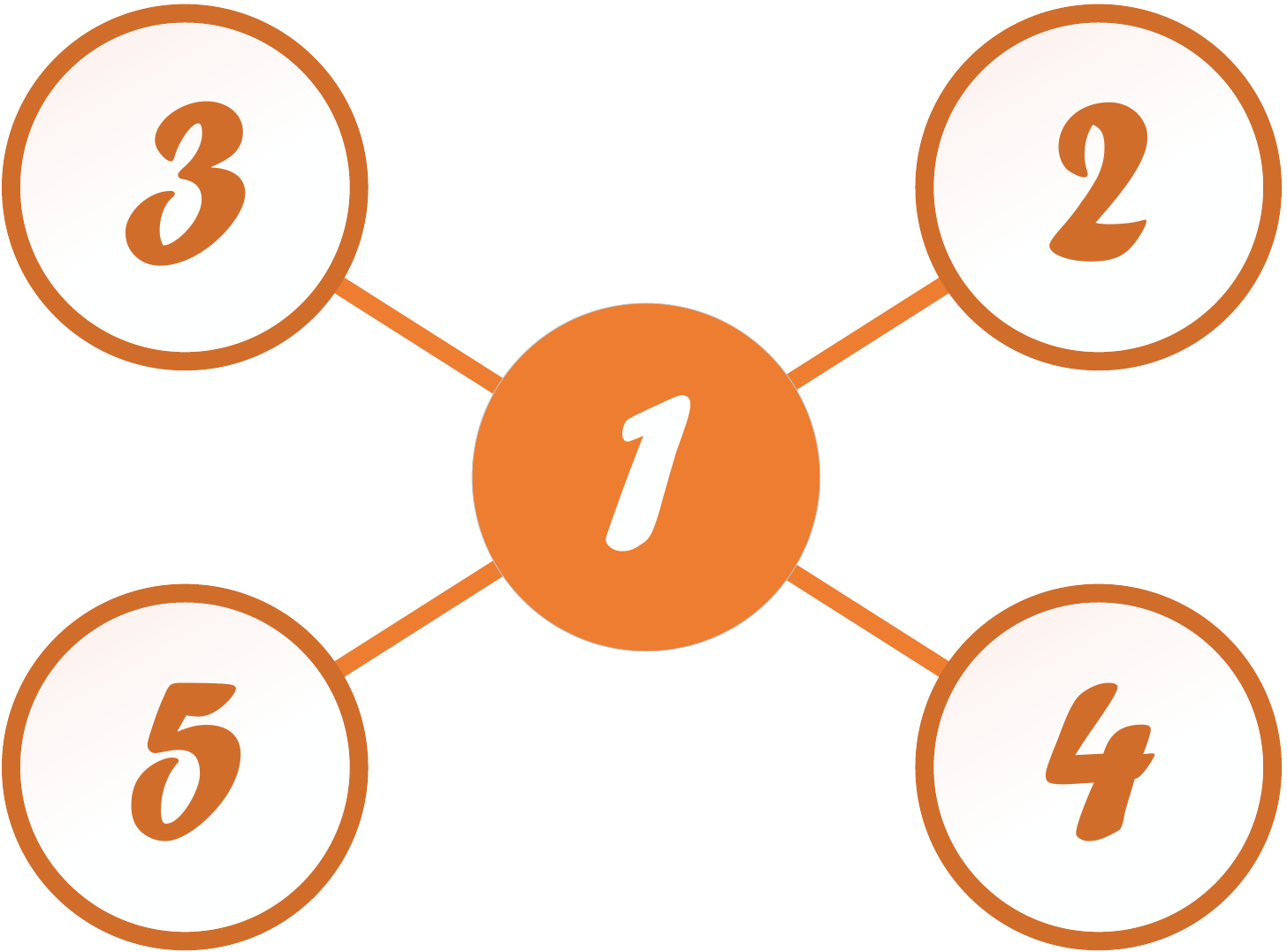}
\label{fig:subgraph}}
\hfil\hspace{-0.5cm}
\subfloat[Constructed tree]{\includegraphics[width=1.8in]{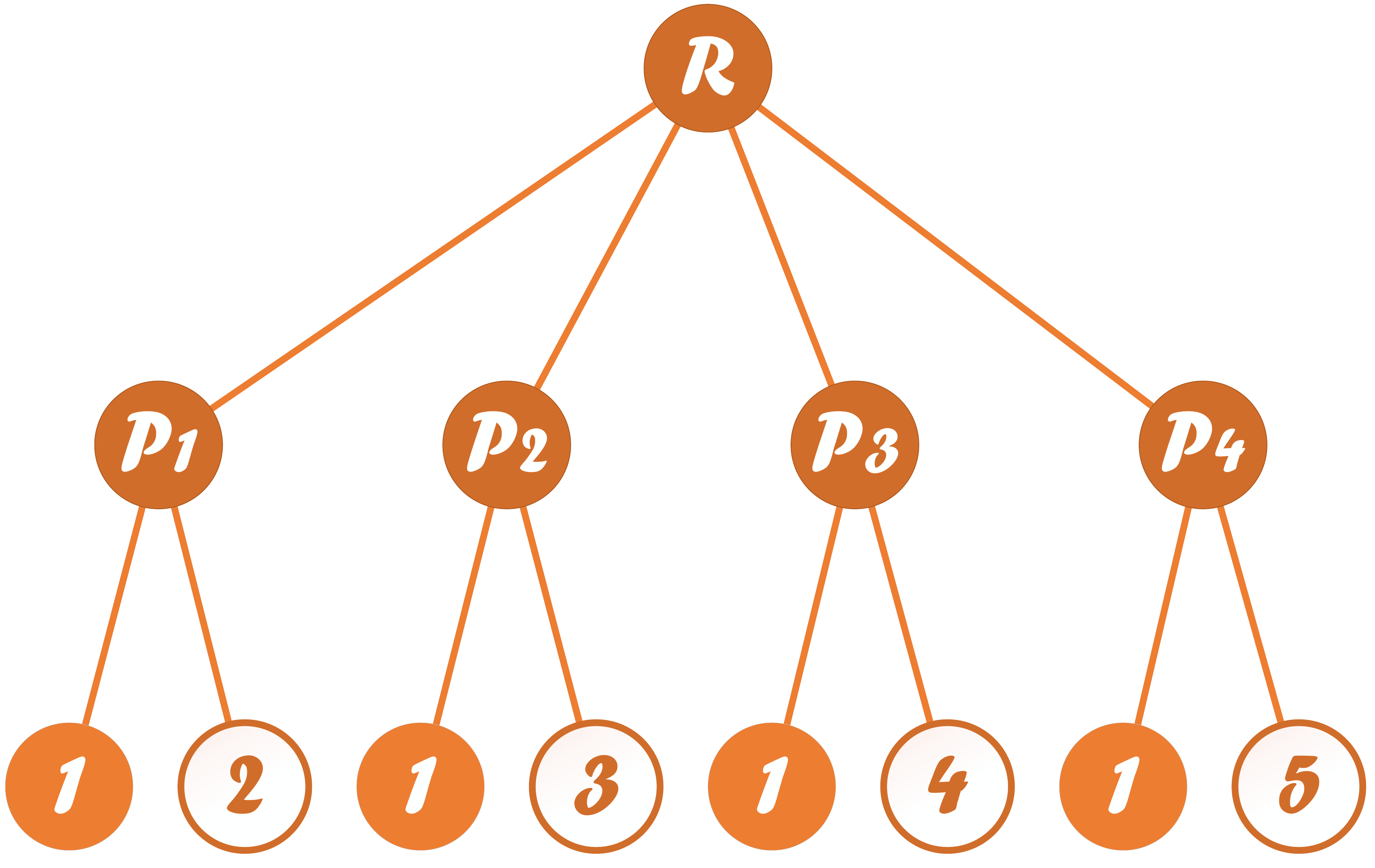}
\label{fig:ctree}}
\caption{Constructed tree for a given ego network}
\label{fig:tree}
\end{figure}
We further use the ego network generated by vertex $1$ in Fig.\ref{fig:tree} as an example to illustrate this tree construction process. First we form 4 leaf pairs $(1,2),(1,3),(1,4),(1,5)$ by joining vertex $1$ and its neighbors together. Then we add four parent nodes $P_1$, $P_2$, $P_3$, and $P_4$ to these leaf pairs accordingly. Lastly, we use a root node $R$ to connect these four parent nodes to build a tree. \textcolor{revise}{Then, we use the same example to explain why \texttt{Lumos} constructs a tree in this way. Since vertex 1 and 2 are symmetric in the subgraph formed with vertex 1, 2, and edge (1,2) in the original ego network, node $P_1$, as a parent of leaves vertex 1 and 2, represents the corresponding subgraph structure accurately. Similarly, nodes $P_2$, $P_3$, and $P_4$ exactly represent their related two-vertex subgraphs. As the root node $R$ is the parent of these four virtual nodes representing all the symmetric two-vertex subgraphs in the ego network, $R$ precisely represents the whole ego network structure.  }

\textcolor{revise}{Although the constructed tree is expressive, it suffers from workload imbalance problem due to degree heterogeneity defined in Definition \ref{ob:dh}.} Some devices become stragglers and slow down the overall learning speed in the distributed system. Worse, it may make one edge device crash down, given that it needs to learn a tree-structured graph of more than thousands of nodes. These potential problems drive us to design a heterogeneity-aware tree constructor.
\subsection{\textcolor{revise}{Workload Balancing Problem Formulation}}
We can use a 0-1 integer programming model to formulate the workload balancing problem. The workload of each device depends on the number of leaves of the tree. The decision variable is a matrix $X\in [0,1]^{|E|}$. 
\begin{equation}
\begin{aligned}
&\text{For any edge } e=(u,v) \in E, \\
&x_e=\left\{ 
\begin{aligned}
    &1,&\text{device $u$ includes neighboring vertex $v$ in its tree}\\
    &0,&\text{otherwise}
\end{aligned}
    \right.
\end{aligned}
\end{equation}
Since the number of leaves in the tree is equal to $2\sum_{e\in E}x_{e}$ for each device due to the leaf pairs introduced above, we define $f(x)$ as the maximum workload in the graph, namely, $f(x)=\max_{u\in V} \sum_{(u,v)\in E}x_{(u,v)}$.
Hence, the \textcolor{revise}{workload balancing} problem is modeled as 
\begin{equation}
     \begin{aligned}
        &\min_{x} f(X)=\min _x\max_{u\in V} \sum_{(u,v)\in E}x_{(u,v)}&\\
        s.t.\quad&x_{(u,v)}+x_{(v,u)}\ge1,\forall (u,v)\in E& \\
        &x_{(u,v)},x_{(v,u)}\in \{ 0,1\},\forall (u,v)\in E& 
    \end{aligned}
    \label{eq:problem}
\end{equation}
The objective function is a min-max function, forbidding the existence of a device that undertakes excessive workload in the whole system. The constraint equation guarantees that every edge is represented in at least one tree. However, this problem is non-trivial to solve. Specifically, we present Theorem \ref{th:np} to show its hardness. 

\begin{theorem}
\label{th:np}
The \textcolor{revise}{workload balancing  problem} defined in (\ref{eq:problem}) is NP-hard. 
\end{theorem}
\begin{proof}

We formulate our problem as a min-max Colored Traveling Salesman Problem (CTSP). In CTSP, there are $m$ salesman and $n$ cities. A digraph $\mathcal{G}=(\mathcal{V},\mathcal{E})$\footnote{We use $G=(V,E)$ to denote the graph in federated GNN and $\mathcal{G}=(\mathcal{V},\mathcal{E})$ to denote the graph in CTSP to distinguish these two graphs.} where $\mathcal{V}=\{ 1,2,\cdots ,n\}$ indicates the traveling route between cities. For edge $(u,v)\in \mathcal{E}$, it has a non-negative weight $w_{u,v}$ quantifying the traveling cost from city $u$ to city $v$. Each salesman is assigned a distinguished color. Denote the color set to be $C=\{ c_1 ,c_2 ,c_3 ,\cdots ,c_m\}$. An indicator matrix called City Color Matrix (CCM) with size $n\times m$ allocates cities with one or more colors. 
\begin{equation}
    CCM_{ij} = \left\{ 
    \begin{aligned}
    &1,&\text{city $i$ has color $c_j$}\\
    &0,&\text{otherwise}
    \end{aligned}
    \right.
    \label{eq.ccm}
\end{equation}
A salesman can only travel to cities with his corresponding color. Each salesman $k\in \mathbb{Z}_m$ starts to travel from a city $s_k$ and returns to the city in the end. The decision variable $X\in \mathbb{R}^{n\times n\times m}$ is a tensor.
\begin{equation}
x_{uvk}=\left\{
    \begin{aligned}
        &1,&\text{Salesman $k$ travels from city $u$ to city $v$}\\
         &0,&\text{otherwise}
    \end{aligned}\right.
\end{equation}
The problem aims to assign different traveling routes to salesmen with the least traveling costs. 



In our problem, we aim to balance the number of vertices to process in each device. Each device represents a salesman, and each edge is a city to visit. We define a mapping function from the edge to the city as $M:E\rightarrow \mathcal{V}$. The graph $\mathcal{G}$ in our scenario is a complete graph where all the vertices are connected. The weights are
\begin{equation}
    w_{ij}=\left\{
    \begin{aligned}
        & 1,  &  M^{-1}(i)\cap M^{-1}(j)\neq \emptyset\\
        & |V|,  & M^{-1}(i)\cap M^{-1}(j)= \emptyset
    \end{aligned}
    \right.
\end{equation}
which heavily punishes assigning a route across an inaccessible city for a salesman. 
Due to privacy concerns, for an edge $(u,v)\in E$, only devices corresponding to vertices $u$ and $v$ are allowed to process the edge. In this way, we assign every salesman with different colors. The starting point of each salesman is randomly assigned from the set of cities he can visit. CCM is defined as
\begin{equation}
    CCM_{ij} =\left\{
    \begin{aligned}
        & 1,  & j \in M^{-1}(i)\\
        & 0,  & j \notin M^{-1}(i)
    \end{aligned}\right.
\end{equation}
In this way, our problem is equivalent to a min-max CTSP problem. It is obvious that the single TSP problem can be reduced to the min-max CTSP problem when $m=1$. Since finding the exact solution to the single TSP problem is proved to be NP-hard\cite{jungnickel1999hard}, our workload balance problem, a min-max CTSP problem, is NP-hard. 
\end{proof}

\subsection{\textcolor{revise}{A Solution to Workload Balancing Problem: Heterogeneity-aware Tree Trimming}} 
Since the problem is NP-hard, finding the optimal solution on real-world, large-scale datasets is infeasible. 
Although it is difficult to find the exact solution, we propose a method to approximate the optimal solution.
Our method consists of two parts, one is a greedy initialization algorithm to trim branches into a sub-optimal tree, and the other is an iterative algorithm based on Markov chain Monte Carlo (MCMC) sampling. We prove this heterogeneity-aware tree trimming method has bounded loss compared to the optimal solution.

\paragraph{Greedy initialization} It is critical to generate an initial solution with relatively small objective values. Otherwise, it will take tremendous time for the following iterative algorithm to converge, particularly in a large-scale global graph. If two vertices are linked, but their degree difference is vast, the device with the more considerable degree trims the branch with the other device to decrease the workload difference. 

The initial solution is defined according to 
\begin{equation}
   x_{(u,v)}=\left\{
    \begin{aligned}
        & 0,  &  \text{round}(\ln \text{deg}(u))\ge \text{round}(\ln \text{deg}(v))\\
        & 1,  & \text{otherwise}\\
    \end{aligned}
    \right.
    \label{eq:is}
\end{equation}
Realizing this simple setting is non-trivial in our scenario, requiring a integer comparison to compare two degrees. Since the degree value reveals structural information of the vertex, it needs privacy preservation.
We use the 2-party protocol for integer comparison problem in CrypTFlow2\cite{rathee2020cryptflow2} due to its efficiency and decentralized setting.  \textcolor{revise}{\label{ga}Other degree protection mechanisms such as graph anonymization \cite{liu2008towards} require complete graph structure on the server. This is inaccessible in our federated scenario because it leaks the local graph neighboring information to the server.}
Taking the logarithm of degrees makes the bits of inter-transmission during secure integer comparison much smaller. It also reduces the unnecessary workload for vertex pairs with a slight degree value difference. The initial solution fills the workload gap between two devices with a significant degree difference. The initialization process can be summarized in Alg.\ref{alg:gi}. 
\textcolor{revise}{Alg.\ref{alg:gi} has time complexity $O\left(\max_v \deg(v) \mathcal{L} \log \mathcal{L}\right)$ where $\mathcal{L}$ is the number of the bits of store each degree value, as within each device, it runs a degree comparison computation for all its neighbours and $O(\mathcal{L} \log \mathcal{L})$ is the complexity of integer comparison\cite{rathee2020cryptflow2}.\label{tc} }
\begin{algorithm}
\caption{Greedy Initialization}
\begin{algorithmic}[1]
\label{alg:gi}
\REQUIRE $G=(V,E)$\\
\ENSURE Vertex sets $(N_1,N_2,\cdots, N_{|V|})$
\FOR{$u \in V$parallel}
\STATE $N_u=\{ \}$
\FOR {$v \in u$'s neighbors}
\IF {round($\ln($deg$(v)))\ge $round$(\ln($deg$(u)))$} 
\STATE $N_u = N_u\cup \{ v\}$
\ENDIF 
\ENDFOR
\ENDFOR
\RETURN $(N_1,N_2,\cdots ,N_{|V|})$
\end{algorithmic}
\end{algorithm}
Alg.\ref{alg:gi} outputs a list of vertex sets whose element $N_u$ is the remaining neighbor set for vertex $u$. The step written in Line 4 utilizes the integer comparison protocol. The output of Alg.\ref{alg:gi} will be updated using an iterative algorithm introduced below. 

\paragraph{Iterative Algorithm}
The algorithm iteratively searches for solutions with smaller objective values using MCMC sampling. We consider each solution as a state and is transitive to another using a transition approach. The transition performs the following operation. For one edge $(u,v)\in E$, and $v\in N_u$
\begin{equation}
N_u=N_u\backslash \{ v\}, N_v=N_v\cup \{u\}.
\end{equation}
It trims the branch including one neighbor $v$ from $\mathcal{T}(u)$ and adds a branch including node $u$ into  $\mathcal{T}(v)$. 

To speed up sampling, we compare the objective values of states before and after $k$ transitions by requesting the devices with the largest workloads given different solutions to compare with each other. Hence, $k$ transitions operate as 
\begin{equation}
    \forall v_i \in \{ v_1,v_2,\cdots ,v_k\}, N_u=N_u\backslash \{ v_i\}, N_{v_i}=N_{v_i}\cup \{u\}
    \label{eq:tk}
\end{equation}
where  $k$ is a random number sampling from $1$ to round$(\ln |N_u|)$. 

The current state $X_t$ will jump to the transited one $X_t'$ with a probability 
\begin{equation}
    \label{eq:tp}
    \begin{aligned}
        &Pr[X_{t+1}=X_t'|X_t]=\min (1, e^{f(X_{t})-f(X_t')})\\
        &Pr[X_{t+1}=X_t|X_t]=1-\min (1, e^{f(X_{t})-f(X_t')}),
    \end{aligned}
\end{equation}
applying the Metropolis-Hastings (MH) algorithm \cite{chib1995understanding} which performs state sampling from distributions of all kinds.  Eq.\ref{eq:tp} allows the state to be transited to a less optimal state with a minuscule probability. Occasionally, a less optimal state will lead to an optimal state with smaller transitions.

To compute $f(X_{t})-f(X')$, we use the 2-party protocol for number computation in CrypTFlow2\cite{rathee2020cryptflow2}.
\textcolor{revise}{The time complexity is $O\left( T|V| \mathcal{L} \log \mathcal{L}\right)$ where $\mathcal{L}$ is the number of the bits to store per workload integer value. For each iteration, the highest cost occurs during execution of Alg.\ref{Alg:argmax}, which needs to run no more than $|V|$ times of integer comparison of time complexity $O(\mathcal{L} \log \mathcal{L})$ on each device in parallel. }


\begin{algorithm}
\caption{MCMC Iteration}
\begin{algorithmic}[1]
\label{alg:mcmc}
\REQUIRE Vertices $V$, number of iterations $T$, initial solution $X_0=(N_1^0,N_2^0,\cdots, N_{|V|}^0)$,\\
\ENSURE Vertex sets $(N_1,N_2,\cdots, N_{|V|})$
\FOR{$t=1\rightarrow T$}
\STATE Find device $u$ with the largest workload under $X_t$ with Alg.\ref{Alg:argmax}
\STATE Device $u$ randomly samples step size $k$
\STATE Device $u$ samples $k$ vertices $\{v_1,v_2\cdots ,v_k\}$ from $N_u$
\STATE Devices $\{u, v_1,v_2\cdots ,v_k\}$ form $X_t'$ using Eq.\ref{eq:tk}
\STATE Find device $u'$ with the largest workload under $X_t'$ with Alg.\ref{Alg:argmax}
\STATE Devices $\{u, u'\}$ compute $f(X_t)-f(X_t')$
\STATE Device $u$ determines $X_{t+1}$ with Eq.\ref{eq:tp}
\STATE Device $u$ sends $X_{t+1}$ to devices $\{v_1,v_2\cdots ,v_k\}$
\ENDFOR
\RETURN $(N_1,N_2,\cdots ,N_{|V|})$
\end{algorithmic}
\end{algorithm}

Steps written in Line 2 and 6 in Alg.\ref{alg:mcmc} require finding the device with the maximum workload. Devices cannot share their workload values directly, similar in greedy initialization. Hence we need to deploy the integer comparison protocol used in Alg.\ref{alg:gi} to develop a system-coordinated module. The algorithm can be divided into two parts. The first part (Line 1-8 in Alg.\ref{Alg:argmax}) is to find a CVS, which is the abbreviation for candidate vertex set, containing all the devices which hold the largest workloads among their neighbors and themselves. Each device first checks whether it is a candidate vertex by comparing its workload with the workloads of its neighboring devices. If it undertakes the largest workload across its ego network, it sends a message to the server. After the server receives all the candidate vertices, it executes the second part. It requests all the devices in CVS to compare themselves with each other. The device with the largest workload among the candidates is expected and informs the server. If more than one device are found, the server randomly selects one as they have the same maximum workloads. Alg.\ref{Alg:argmax} presents the procedure. 
\begin{algorithm}
\caption{Finding the device with the maximum workload}
\begin{algorithmic}[1]
\label{Alg:argmax}
\REQUIRE Workload \\
$(w_1,w_2,\cdots ,w_{|V|})=(|N_1|,|N_2|,\cdots ,|N_{|V|}|)$
\ENSURE A device $v$ with the maximum workload\\
\textbf{Server operation}
\STATE Candidate Vertex Set $CVS=\emptyset$
\FOR{$v \in V$ parallel}
\STATE $ Candidate_v$ = Execute \textbf{Device $v$ operation 1}
\ENDFOR\\
\FOR {$v \in V$}
\IF {$Candidate_v$}
\STATE $CVS=CVS\cup \{ v\}$
\ENDIF
\ENDFOR
\FOR{$v \in CVS$ parallel}
\STATE Execute \textbf{Device $v$ operation 2} with $CVS$
\ENDFOR
\STATE Return the device $v$ with the maximum workload\footnotemark{}
\quad \\
\textbf{Device $v$ operation 1}\\
\STATE Compare workloads in $\mathcal{N}(v)$ with $w_v$ one by one
\STATE $Candidate=$ whether $w_v$ is the largest

\STATE Send $Candidate$ to the server
\quad \\
\textbf{Device $v$ operation 2}\\
\STATE Compare workloads in $CVS$ with $w_v$ one by one
\STATE Send the server a message whether $w_v$ is the largest 
\end{algorithmic}
\end{algorithm}
\footnotetext{If more than one devices are found, the server randomly selects one. }


\begin{theorem}
\label{theo:mcmc}
The probability of sampling a solution with $f(X)=f(X_{opt})-c|E|\ln 2$ decreases with exponential in $|E|\ln 2$. 
\begin{equation}
    Pr\left[f(X)\leq f(X_{opt})-c|E|\ln 2\right]\leq e^{-|E|\ln 2}
\end{equation}
\end{theorem}
\begin{proof}

Let $m\in [0,1]$. We have 
\begin{equation}
    Pr\left[f(X)\leq m \right]\leq \frac{Pr\left[f(X)\leq m \right]}{Pr\left[f(X)=f(X_{opt}) \right]}
\end{equation}
 Since every state can be obtained through a sequence of transitions from another state, the Markov Chain is connected. Since $Pr\left[ f(X)\right] >0$, the Markov chain is ergodic and reversible\cite{upfal2005probability}. The standard MH procedure and the property of the Markov Chain guarantee that the final solution will converge to the stationary distribution. Hence,  
 \begin{equation}
 \begin{aligned}
     Pr\left[f(X)\leq m \right]&\leq \frac{Pr\left[f(X)\leq m \right]}{Pr\left[f(X)=f(X_{opt}) \right]}\\
     &\leq \frac{N\frac{e^m}{\sum_X e^{f(X)}}}{\frac{e^{f(X_{opt})}}{\sum_X e^{f(X)}}} \leq Ne^{m-f(X_{opt})}
 \end{aligned}
\end{equation}
where $N$ is the number of states in the Markov Chain. In our scenario, $N=2^{c|E|}$ where $c$ is a constant. Let $m=f(X_{opt})-\ln (N)=f(X_{opt})-(\ln 2)c|E|-t$ where $t\in \mathbb{R}$. Then, 
\begin{equation}
\begin{aligned}
     &Pr\left[f(X)\leq f(X_{opt})-(\ln 2)c|E| -t \right]\\
     \leq& 2^{c|E|}e^{-(\ln 2)c|E|-t}=  e^{-t}.
\end{aligned}
\end{equation}
Let $t=(\ln 2)|E|$. We have 
\begin{equation}
    \begin{aligned}
     &Pr\left[f(X)\leq f(X_{opt})-c'\ln 2|E|  \right]
     \leq & e^{-|E|\ln 2}
\end{aligned}
\end{equation}
where $c'=c+1$. This completes the proof. 
\end{proof}
Theorem \ref{theo:mcmc} theoretically bounds the probability of obtaining a solution far from the optimal using our approach. 
\section{Tree-based Graph Neural Network Training\label{sec:gnn}}
In this section, we introduce the second module of \texttt{Lumos}, graph neural network trainer. 
We describe embedding initialization,  message passing and loss computing, all key components in GNN training, in details sequentially. 
In general, after balancing the computation workloads for each device, each device now has a tree-structured graph of an adequate size, which allows graph neural network training. Every device performs message passing on its tree, including all the leaves and internal nodes. It uses an unbiased encoder to encrypt local features to protect $\epsilon-$ local differential privacy. The system combines all the embeddings of leaves corresponding to the same vertex through inter-device communications to obtain the final vertex embeddings. These embeddings are then used to compute loss values to update model weights. 
\subsection{LDP Embedding Initialization}
\label{encoder}
In centralized GNN, the initial embedding of vertex $u$ is
\begin{equation}
    \mathbf{h_u^0}=\left\{
    \begin{aligned}
        & \mathbf{x_u},  &  \text{node has features}\\
        & \mathbf{0},  & \text{otherwise}\\
    \end{aligned}
    \right.
\end{equation}
Since for the tree-structured graph, the leaves correspond to the original vertex in the global graph, they are thus featured. The internal nodes are virtual nodes to facilitate training without features. Hence, we define the initial embedding of node $\mu $ in the tree as 
\begin{equation}
    \mathbf{h_\mu^0}=\left\{
    \begin{aligned}
        & \mathbf{x_u},  &  \text{node }\mu\text{ is a leaf related to vertex $u\in V$}\\
        & \mathbf{0},  & \text{otherwise.}\\
    \end{aligned}
    \right.
\end{equation}
Nevertheless, privacy concerns impede the direct publication of node features. Node features need encryption without losing much utility. We take advantage of the one-bit mechanism \cite{ding2017collecting} to send local features to other devices. Assuming that $\mathbf{x_u}\in [a,b]^d$, for the $i$th element, the probability to map the element value to $\{0,1\}$ is
\begin{equation}
\begin{aligned}
    Pr(x_{u,i}'=1)&=\frac{1}{e^{\frac{\epsilon\text{wl}(u)}{d}}+1}+\frac{x_{u,i}-a}{b-a}\frac{e^{\frac{\epsilon\text{wl}(u)}{d}}-1}{e^{\frac{\epsilon\text{wl}(u)}{d}}+1}\\
    Pr(x_{u,i}'=0)&=1-Pr(x_{u,i}'=1),
    \end{aligned}
\end{equation}
where $\epsilon$ is a hyper-parameter called privacy budget to control privacy protection and $\text{wl}(u)$ is the workload value of degree u, i.e., the number of neighbours remaining in the trimmed tree $\mathcal{T}(u)$. A smaller $\epsilon$ protects more privacy despite more information loss. 
The mapping probability $Pr(x_{u,i}'=1)$ is large when the element value is close to $b$. On the other hand, $Pr(x_{u,i}'=0)$ is large when it is close to $a$. 

Naively encoding all the feature elements with the 1-bit mechanism successfully transmits all the dimensions of local feature to its neighboring clients to some extent. Yet, every featured element suffers a huge utility loss , making the final embeddings deviate far from the optimal. Hence, the feature encoder only encodes part of them while abandoning the rest. Specifically, the feature encoder distributes the $d$ encoded elements into $\text{wl} (u)$ bins randomly. It sends a partial encoded feature only containing the elements in the $k$th bin to the $k$th neighbor respectively for all $k\in \mathbb{Z}_{\text{wl} (u)}$. Distributing encoded elements ensures that all the feature information are sent to one of its neighbors for further aggregation. The partial feature fills the missing elements with 0.5, implying no deviation towards the maximum or minimum value and transmitting no information. Compared with encoding all the elements, the encoded features protect privacy by sending part of data under the same budget. \textcolor{revise}{Moreover, since parts of elements are constant, encoded features have lower variance compared to sending fully encoded features, resulting in a better model.\label{partial} }


 After the devices compute the encoded features $\mathbf{x_u'}\in \{0,0.5, 1\}^d$, they calculate the recovered features $\mathbf{x_u''}$ from 
\begin{equation}
    x_{u,i}''=\left\{
    \begin{aligned}
       & \frac{b-a}{2}\frac{e^{\frac{\epsilon\text{wl}(u)}{d}}+1}{e^{\frac{\epsilon\text{wl}(u)}{d}}-1}+\frac{a+b}{2},&x_{u,i}'&=1\\
        &\frac{a-b}{2}\frac{e^{\frac{\epsilon\text{wl}(u)}{d}}+1}{e^{\frac{\epsilon\text{wl}(u)}{d}}-1}+\frac{a+b}{2},&x_{u,i}'&=0\\
        &\frac{a+b}{2},&x_{u,i}'&=0.5\\
    \end{aligned}
    \right.
    \label{eq:x''}
\end{equation}
It is worth noting that the recovered features are not equal to the original features. They only reveal the tendency of the feature values towards the lower bound $a$ or the upper bound $b$. However, we can prove that such a recovery process do not incur too much utility loss using the following theorem.
\begin{theorem}
The recovered features are unbiased. 
\label{theo:unbiased}
\end{theorem}
\begin{proof}
Let $\epsilon'=\frac{\epsilon\text{wl} (u)}{d}$.
We know that $x_{u,i}'$ follows a Bernoulli distribution. In this way, 
\begin{equation}
    E[x_{u,i}']= Pr(x_{u,i}=1)=\frac{1}{e^{\epsilon'}+1}+\frac{x_{u,i}-a}{b-a}\frac{e^{\epsilon'}-1}{e^{\epsilon'}+1}.
\end{equation}
We can rewrite Eq.\ref{eq:x''} into 
\begin{equation}
     x_{u,i}''=(x_{u,i}'-\frac{1}{2})(b-a)\frac{e^{\epsilon'}+1}{e^{\epsilon'}-1}+\frac{a-b}{2}.
\end{equation}
Then, 
\begin{equation}
\begin{aligned}
    E[x_{u,i}''] =& (\frac{1}{e^{\epsilon'}+1}+\frac{x_{u,i}-a}{b-a}\frac{e^{\epsilon'}-1}{e^{\epsilon'}+1}-\frac{1}{2})(b-a)\frac{e^{\epsilon'}+1}{e^{\epsilon'}-1}\\
    &+\frac{a+b}{2}\\
    =&\frac{b-a}{e^{\epsilon'}-1}+x_{u,i}-a-\frac{(b-a)(e^{\epsilon'}+1)}{2(e^{\epsilon'}-1)}+\frac{a+b}{2}\\
    =&x_{u,i}
\end{aligned}
\end{equation}
which completes the proof. 
\end{proof}
The unbiasedness of the recovered features prevents significant utility loss of the feature encoder.  Upon recovering features, each device sends them to neighbouring devices and starts training based on the following message passing design.
\subsection{Message Passing}
Each device considers its tree as a graph and performs message passing along the tree branches. It updates the node embeddings in its device using $l$ graph neural network layers. A tree-based GNN layer is defined as Eq.\ref{eq:agg}. After an $l$-layer update, each device sends the embedding of leaves corresponding to its neighboring vertices to their own devices. 

Upon receiving the updated embeddings of different leaves, the device combines them to generate the final vertex embedding using a pooling function. \textcolor{revise}{\label{pool}The pooling function allows the message to pass across multi-hop neighbors as the embeddings before pooling are generated from different trees, revealing neighboring information of different vertices. }We use an average pooling function in the experiment. 
\begin{equation}
    \mathbf{h_u} = \text{POOL}(\{ \mathbf{h_\mu^l}|\text{leaf }\mu \text{ correspond to vertex }u\})
\end{equation} 
\subsection{Loss Computing}
Device $u$ now obtains its own vertex embedding $\mathbf{h_u}$. Then it needs to calculate the loss value to operate backward propagation. GNN performs well in supervised learning and unsupervised learning scenarios. Our system is compatible in both settings. 
\paragraph{Supervised learning}
Device $u$ is aware of its label $y_u\in \mathbb{Z}_L$. Locally it can use a linear layer to convert the $d$ dimensional embedding $\mathbf{h_u}$ into an $L$ dimensional vector $\mathbf{z_u}$.
\begin{equation}
    \mathbf{z_u} = \text{LINEAR}(\mathbf{h_u})
\end{equation}
It then uses an activation function such as softmax to convert the embedding into a probability vector $\mathbf{p_u}\in [0,1]^L$. The device can now use the classical cross-entropy loss function to compute the loss. 

\paragraph{Unsupervised learning}
Unsupervised learning requires to embed nodes without any label data. However, the model needs to update its weights by evaluating how representative its current embeddings are. To this end, \texttt{Lumos} performs a link prediction task to compute the loss values as a compensation for the non-existence of labeling data since device $u$ is aware of its neighboring vertices $\mathcal{N}(u)$.  Device $u$ requests all the selected vertices $N_u$ to send their embeddings $\{ \mathbf{h_v} | v\in N_u\}$. Then device $u$ negatively samples vertices that are not its neighbors in the global graph and requests their embeddings $\{ \mathbf{h_{v'}}\}\sim P_{\{ \mathbf{h_v}|(u,v)\notin E\}}$. The network prefers proximate embeddings for neighboring vertices and distinct for distanced vertices. In this way, we can compute the loss value.
\begin{equation}
\begin{aligned}
    &\mathcal{L}(\mathbf{h_u}) \\ =&-\sum_{\mathbf{h_v}\in \{ \mathbf{h_v} | v\in N_u\}}\log (\sigma( \mathbf{h_u}\mathbf{h_v}))-\sum_{\mathbf{h_v}\in \{ \mathbf{h_{v'}}\}}\log (-\sigma (\mathbf{h_u}\mathbf{h_v}))
\end{aligned}
\end{equation}

After computing the loss value of each node in each device, the devices share their loss values with each other and aggregate them to get the final loss value. Then all the devices can operate backward propagation to update the weights in the network. The training process continues until the network model converges.
\section{\textcolor{revise}{Privacy Analysis}}
\textcolor{revise}{In this section, we present a privacy analysis to show that \texttt{Lumos} meets the privacy requirement restrictively.\label{pri_proof} }
\textcolor{revise}{First, we analyze the privacy protection of features. In \texttt{Lumos}, the only module to transmit local features is the embedding initiator. Theorem \ref{theo:dp} states that the encoding before feature exchange is a $\epsilon$-local differential privacy mechanism. } 

\begin{theorem}
The embedding initialization protects $\epsilon$-local differential privacy. 
\label{theo:dp}
\end{theorem}
\begin{proof}
We deploy a 1-bit encoder and a bijective recovery process to every selected feature bit with the noise parameter $\frac{\epsilon \text{wl}(u)}{d}$. The 1-bit encoder is proved to protect $\epsilon-$ local differential privacy\cite{ding2017collecting}. The composability property of LDP \cite{mcsherry2009privacy} ensures that a series of randomized mechanisms with $\epsilon_i$-LDP ensures $\sum_i \epsilon_i$-DP in total. Therefore, the feature encoder satisfies $\frac{d}{\text{wl}(u)}\cdot \frac{\epsilon\text{wl}(u)}{d}=\epsilon$-DP. The recovery function is bijective, so it doesn't impact the $\epsilon-$LDP protection of the encoder. Hence, the embedding initialization protects $\epsilon-$local differential privacy. 
\end{proof}

\textcolor{revise}{Next, we analyze the privacy protection of node degrees. In \texttt{Lumos}, degree privacy protection follows zero-knowledge protocol as required. }
\begin{theorem}
\textcolor{revise}{Lumos protects node degree  under a zero-knowledge protocol. }
\label{theo:zkp}
\end{theorem}
\begin{proof}
\textcolor{revise}{\texttt{Lumos} deploys a 2-party encryption protocol in CrypTFlow2\cite{rathee2020cryptflow2} to protect the original node degree and make integer comparison computations. The encryption protocol is a hybrid algorithm of multiple oblivious transfer steps\cite{brassard1986all}. In \cite{kilian1988founding}, it is proved that oblivious transfer guarantees a zero-knowledge proof. }
\end{proof}

\section{Evaluation\label{sec:ev}}
In this section, we present the experimental results and analysis of \texttt{Lumos}. We first examine how \texttt{Lumos} performs compared to baseline systems in supervised and unsupervised scenarios. Then we perform sensitivity analysis of privacy parameter $\epsilon$. Last but not least, we conduct ablation study by replacing two key components in \texttt{Lumos} to show their irreplaceability in terms of accuracy and system cost.

\subsection{Dataset}

We use two social network graphs with different sizes: Facebook\cite{rozemberczki2021multi} and LastFM\cite{rozemberczki2020characteristic}. We split the graphs into $|V|$ ego networks so that each device represented by one vertex in the graph holds its corresponding ego network. The ego networks only include the feature and label for the central vertex without any node information related to other vertices. 
\begin{itemize}
    \item Facebook page-page graph describes the interaction among different Facebook pages as a representation of online social relationships. Its 22,470 vertices are all Facebook pages, and 170,912 edges represent a mutual ``like". 4,714 vertex features are obtained from the page descriptions. Its label is the category of the pages. There are 4 categories. 

\item LastFM graph is crawled from a music streaming website, LastFM. Its 7,624 vertices are users, and the 55,612 edges represent ``following" relationship between them. Its 128 vertex features are users' preferred musicians. Its label is the nationality of the user. There are 18 classes. 
\end{itemize}
\subsection{Experimental Settings}
 In the tree constructor, we run \textcolor{revise}{1,000 MCMC iterations for Facebook dataset and 300 for LastFM dataset} to balance workloads and construct trees. We utilize GCN \cite{welling2016semi} and GAT\cite{velivckovic2018graph} as two backbone GNN layers to test the compatibility of our system. All the GNN models have $l=2$ layers. The GNN layers are followed by an activation ReLU function and a dropout function with a probability 0.01. The GAT layer has four attention heads. The output dimension and the hidden dimension are 16. All the models are trained using the Adam optimizer. The learning rate $lr$ is set to be $0.01$. The privacy parameter $\epsilon$ is set to be $2$. The GNN trainer trains the model for 300 epochs. The experiments are done with an Intel Core i7-10700 CPU clocked at 2.90GHz, an 8GB-memory NVIDIA RTX 3070, and 16GB memory. 

We first evaluate the performance of our system in supervised settings to complete a label classification task. We uniformly sample the vertices into training, validation, and test sets for all the datasets with 50\%, 25\%, and 25\% ratios. The label with the maximum probability is chosen and compared with the ground truth to compute the classification accuracy.

Then we conduct experiments in unsupervised settings to complete a link prediction task. We uniformly sample the edges into training, validation, and test sets for all the datasets with 80\%, 5\%, and 15\% ratios. Each link is predicted with a given probability and compared with the ground truth to compute the ROC-AUC score \cite{fawcett2006introduction}, which is the probability that a classifier will assign a randomly chosen positive instance with a higher probability of being positive than a randomly chosen negative instance. We prefer models with higher scores.
\subsection{Comparison Methods}
 We compare our methods with centralized GNN network models, a local differential private GNN system (LPGNN\cite{sajadmanesh2021locally}), and a naive federated GNN system. 

 \begin{itemize}
\item Centralized GNN network models take a complete global graph as input to a GNN model. Both the edge information and the node features are accessible to the server. 
 
 \item LPGNN protects node features and labels with locally $\epsilon_x$ and $\epsilon_y$  differential privacy accordingly. \textit{However, it assumes that the server owns the graph structure.} Although it is proposed in an environment where privacy is less restricted than ours, we can still compare the model's prediction performance with \texttt{Lumos}. In the experiments, we set $\epsilon_x=2$ and $\epsilon_y=1$. Since LPGNN is proposed in supervised settings, we do not measure its performance in unsupervised scenarios.
 
 \item Naive FedGNN system noises graph statistics to ensure local differential privacy. It uses Gaussian Mechanism\cite{dwork2014algorithmic} to noise features and Randomized Response\cite{warner1965randomized} to noise adjacency matrix and labels. The devices send the noised ego networks to the server to train GNN upon the noised graph. 
 
 
 \end{itemize}
\subsection{Performance Results}
\subsubsection{Supervised Learning}

 The results are shown in Fig.\ref{fig:perfosu}. \texttt{Lumos} merely loses a 16.29\% accuracy upon Facebook dataset and 14.73\% upon LastFM dataset compared to centralized GCN. On the other hand, \texttt{Lumos} merely loses a 16.30\% accuracy upon Facebook dataset and 15.76\% upon LastFM dataset compared to centralized GAT. This shows that \texttt{Lumos} performs close to centralized GNN.  On Facebook dataset, \texttt{Lumos} outperforms LPGNN with a 6.44\% accuracy increase for GCN and a 5.10\% increase for GAT. On LastFM dataset, \texttt{Lumos} outperforms LPGNN with a 12.26\% accuracy increase for GCN and an 7.62\% increase for GAT. The result shows that \texttt{Lumos} embeds vertices more representative than LPGNN does with more privacy restrictions, particularly for datasets with excessive classes since LPGNN noises label to upload them to the server while \texttt{Lumos} processes original labels locally. \texttt{Lumos} performs much better than naive FedGNN does. On Facebook dataset, \texttt{Lumos} outperforms naive FedGNN with a relatively 74.29\% accuracy increase for the GCN model and 42.13\% for the GAT model. On LastFM, \texttt{Lumos} outperforms naive FedGNN with a relatively 52.85\% accuracy increase for GCN model and 32.62\% for the GAT model. This imply that our \texttt{Lumos} resolves the issue caused by federation of graph. 

\begin{figure}[!h]
\centering
\subfloat[Facebook]{\includegraphics[width=1.6in]{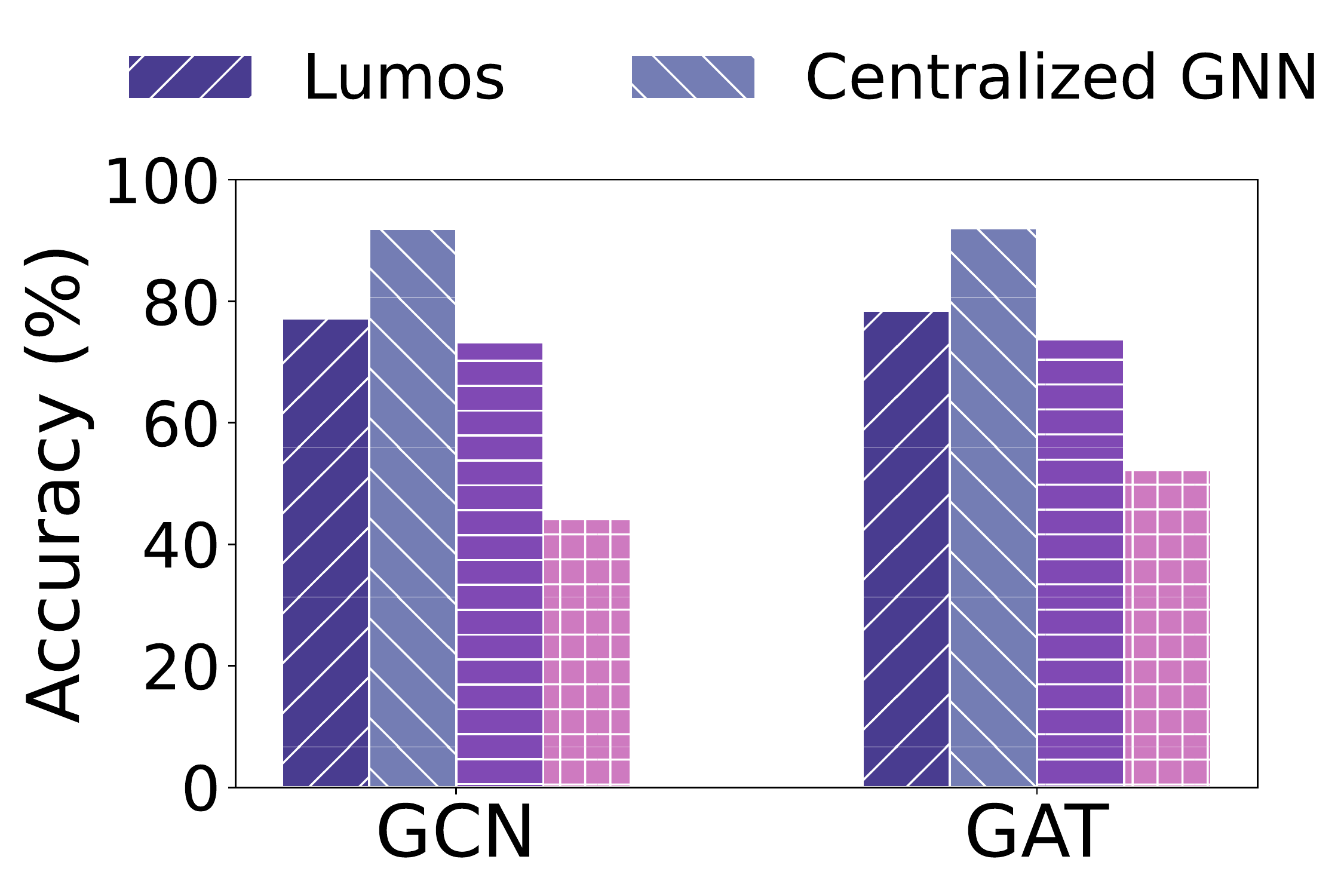}
\label{fig:fbsu}}
\hfil
\hspace{-0.5cm}
\subfloat[LastFM]{\includegraphics[width=1.6in]{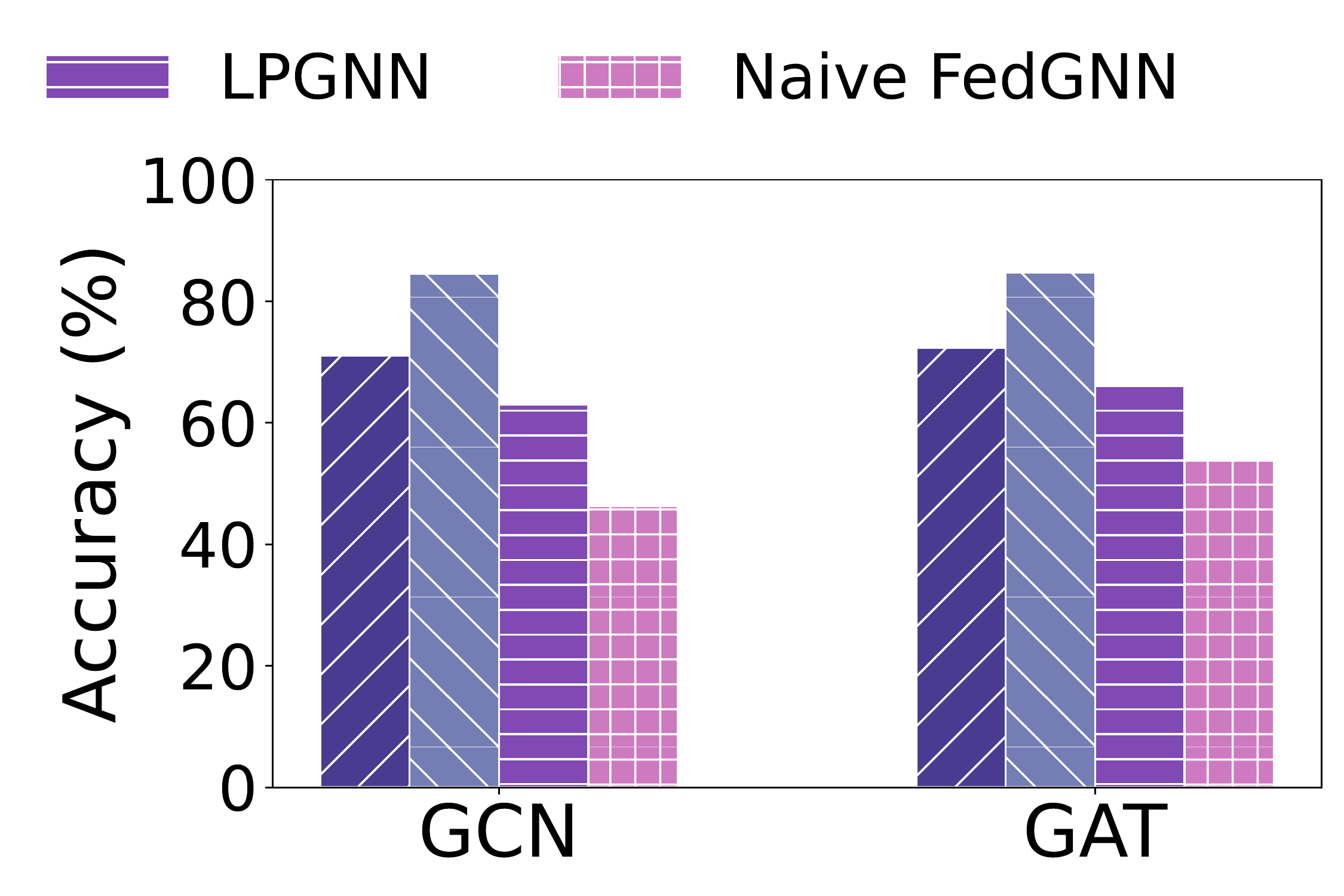}
\label{fig:fmsu}}
\caption{Label classification accuracy}
\label{fig:perfosu}
\end{figure}
\subsubsection{Unsupervised Learning}
Fig.\ref{fig:perfoun} presents the results. \texttt{Lumos} merely loses a 7.13\% AUC score upon Facebook dataset and a 3.60\% AUC score upon LastFM compared to centralized GCN. On the other hand, \texttt{Lumos} merely loses a 9.12\% AUC score upon Facebook dataset and a 4.59\% AUC score upon LastFM compared to GAT. \texttt{Lumos} performs much better than naive FedGNN does. On Facebook dataset, \texttt{Lumos} outperforms naive FedGNN with a relatively 22.94\% AUC score increase for the GCN model and 22.03\% for the GAT model. On LastFM, \texttt{Lumos} outperforms naive FedGNN with a relatively 19.72\% accuracy increase for GCN model and 19.60\% for the GAT model. This imply that our \texttt{Lumos} resolves the issue caused by federation of graph. 
\begin{figure}[!h]
\centering
\subfloat[Facebook]{\includegraphics[width=1.6in]{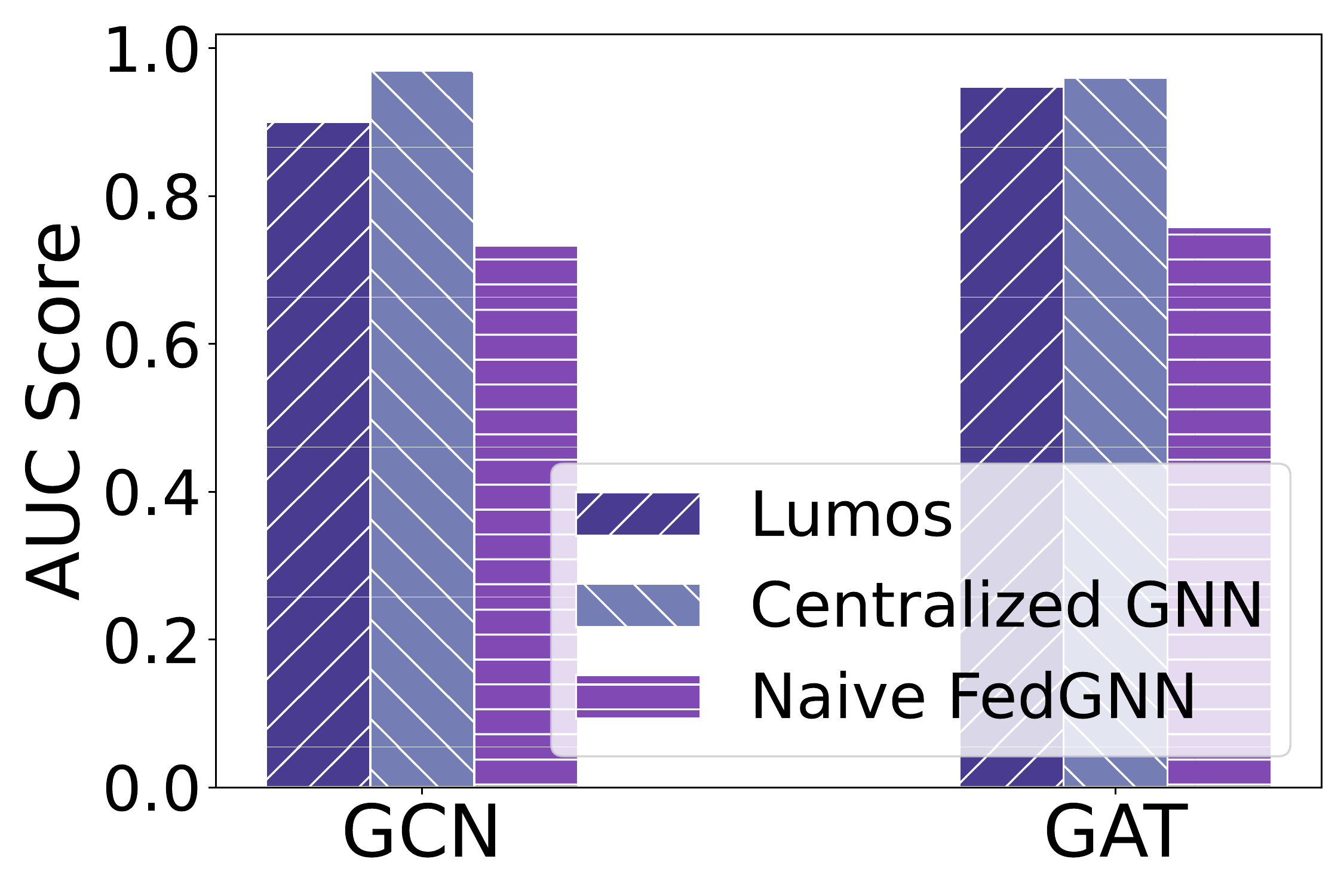}
\label{fig:fbun}}
\hfill
\hspace{-0.5cm}
\subfloat[LastFM]{\includegraphics[width=1.6in]{unsupervised_facebook.pdf}
\label{fig:fmun}}
\caption{Link prediction ROC-AUC score}
\label{fig:perfoun}
\end{figure}



\subsection{Sensitivity Analysis of Privacy Parameter $\epsilon$}
Figure \ref{fig:sense} illustrates the effect of privacy parameter $\epsilon$ upon the model performance of \texttt{Lumos}. For supervised learning, raising $\epsilon$ from 0.5 to 1, 2, and 4 results in relatively 5.13\%, 9.23\%, and 10.26\% accuracy increases on Facebook dataset and relatively 9.24\%, 14.85\% and 16.82\% accuracy increases on LastFM dataset. For unsupervised learning, raising $\epsilon$ from 0.5 to 1, 2, and 4 results in relatively 8.88\%, 15.04\%, and 17.14\% AUC score increases on Facebook dataset and 9.59\%, 16.37\% and 18.93\% AUC score increase on LastFM dataset. Both plots show that \texttt{Lumos} with larger $\epsilon$ performs better as the feature encoder with small $\epsilon$ protects more privacy by deviating the output features more from the original. Moreover, \texttt{Lumos} is robust to variation in large $\epsilon$ values.

\begin{figure}[!h]
\centering
\subfloat[Supervised]{\includegraphics[width=1.6in]{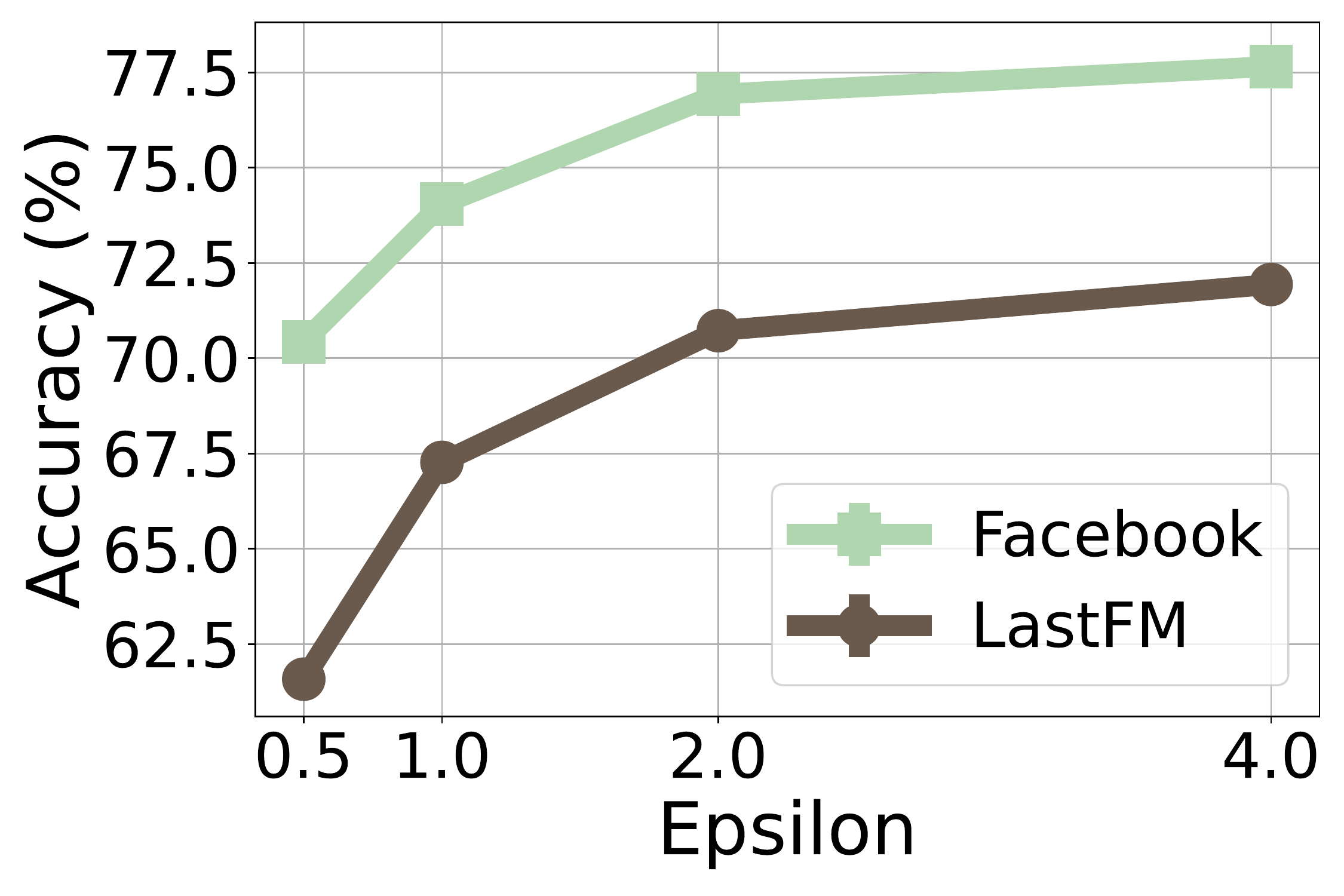}
\label{fig:sense_super}}
\hfil
\hspace{-0.5cm}
\subfloat[Unsupervised]{\includegraphics[width=1.6in]{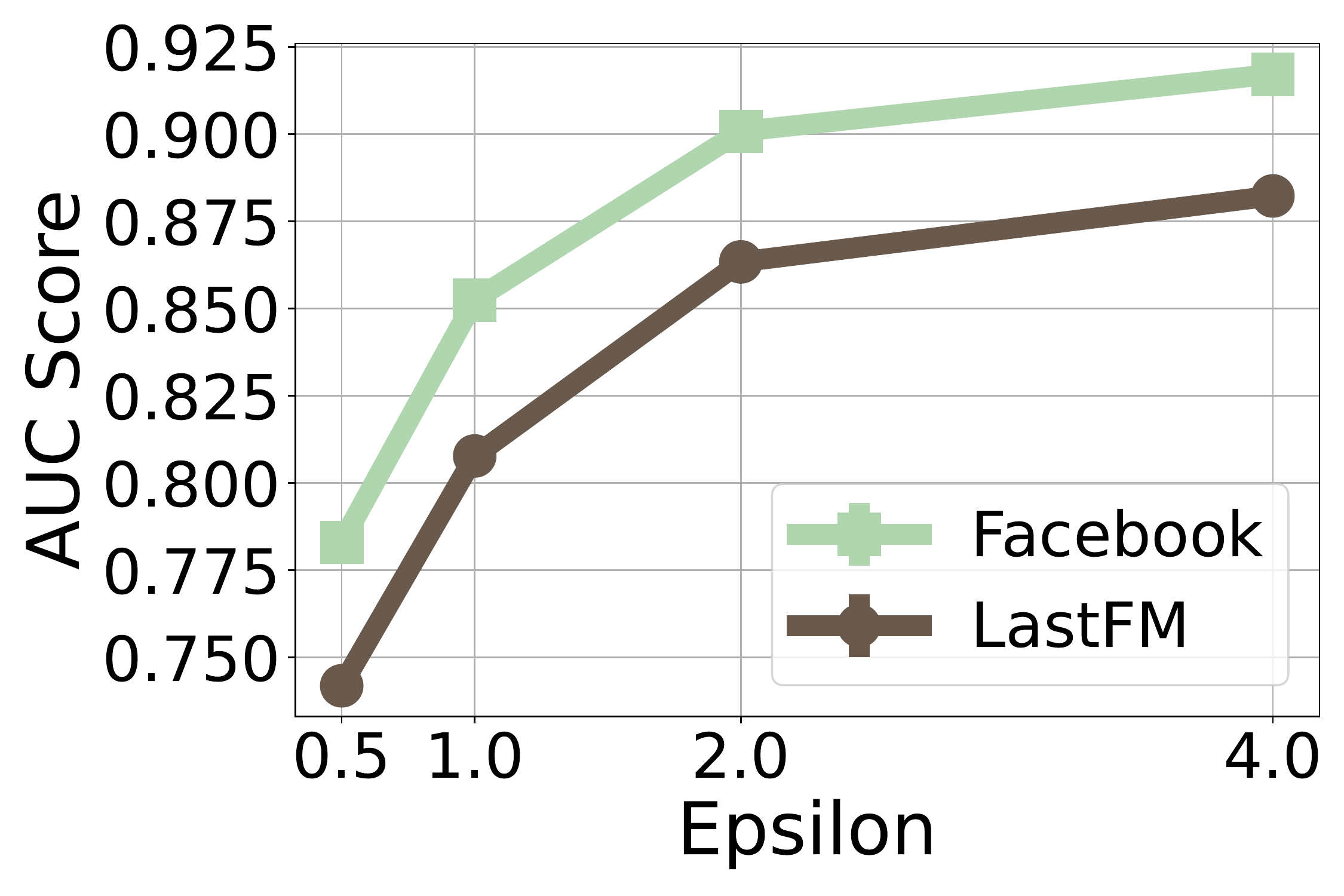}
\label{fig:sense_un}}
\caption{Effect of privacy parameter $\epsilon$}
\label{fig:sense}
\end{figure}

\subsection{Ablation Study}
We perform the ablation study by replacing two essential components in \texttt{Lumos} with naive implementation to demonstrate their effectiveness. \texttt{Lumos} without virtual nodes, referred to as \texttt{Lumos} w.o. VN, takes the original ego network as the input to the tree-based GNN trainer. \texttt{Lumos} without tree trimming, referred to as \texttt{Lumos} w.o. TT, uses all the vertices in the ego network to construct a tree by skipping the tree trimming process. 
\subsubsection{Effect of two modules on accuracy}
Fig.\ref{fig:ab} presents the accuracy performance comparsion of \texttt{Lumos} with these two approaches.
\texttt{Lumos} without virtual nodes performs the worst. In supervised learning, \texttt{Lumos} outperforms the one without virtual nodes with a 13.73\% accuracy increase for the GCN model, 13.10\% for the GAT model on Facebook dataset, and a 16.36\% accuracy increase for the GCN model and 15.70\% for GAT model on LastFM.  In unsupervised learning, \texttt{Lumos} outperforms the one without virtual nodes with a 8.26\% AUC score increase for the GCN model, 8.36\% for the GAT model on Facebook, and a 8.99\% AUC score increase for the GCN model and 7.69\% for the GAT model on LastFM. The figures imply that \texttt{Lumos} successfully improves the representation by adding virtual nodes and enlarging the local graphs. \texttt{Lumos} performs very close to \texttt{Lumos} without tree trimming with relative accuracy and AUC score differences less than 0.01\%. The slight  differences indicate that \texttt{Lumos} is still expressive despite tree trimming since it ensures that every neighboring pair exists in at least one tree. 

\begin{figure}[!h]
\centering
\subfloat[Supervised]{\includegraphics[width=1.6in]{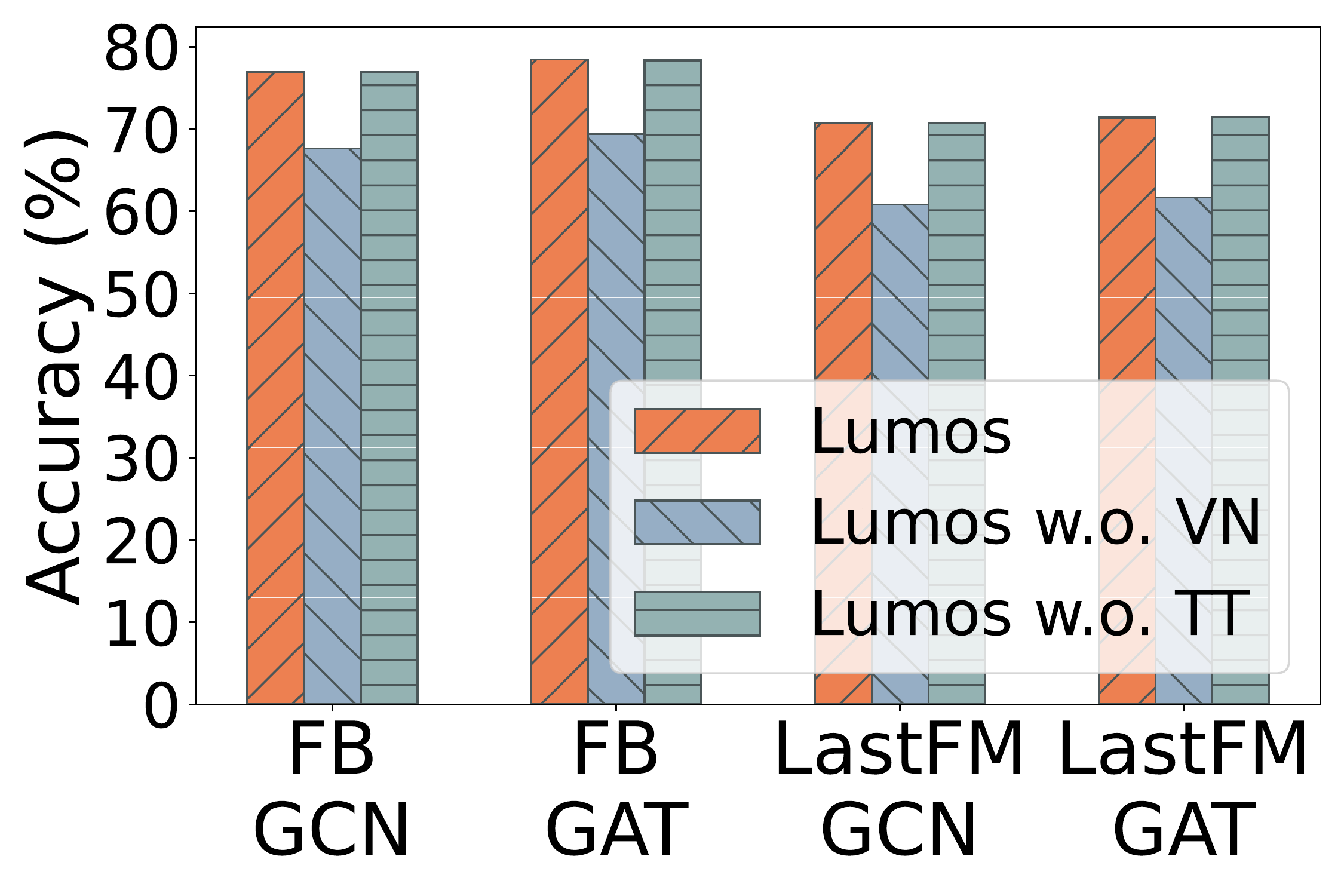}
\label{fig:ab_super}}
\hfil\hspace{-0.5cm}
\subfloat[Unsupervised]{\includegraphics[width=1.6in]{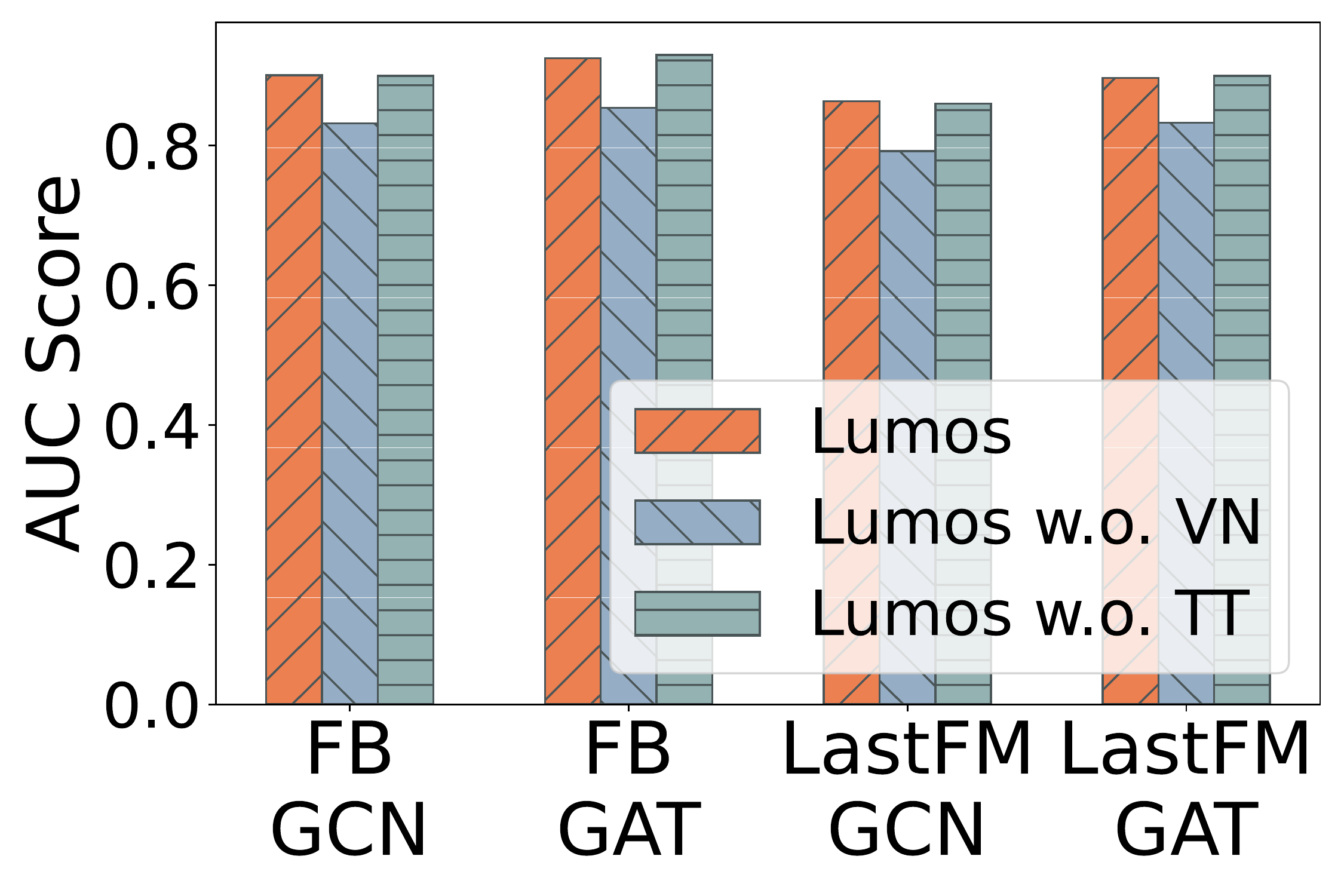}
\label{fig:ab_un}}
\caption{Accuracy contribution of each module in Lumos}
\label{fig:ab}
\end{figure}

\subsubsection{Effectiveness in balancing the workload}
Tree constructor influences not only embedding representativeness but also systematic performances by trimming trees to balance workloads across devices. Therefore, we specifically reveal its effect on system performance.
We compare the workload distribution with and without tree trimming in Facebook and LastFM datasets. We measure workload using the number of selected neighboring vertices in the device. Without trimming, the workload is the original degree value in the graph. Fig.\ref{fig:fl} illustrates that the CDF of workloads of \texttt{Lumos} no longer has a heavy tail which \texttt{Lumos} without tree trimming has. In Facebook dataset, the maximal workload with tree trimming is 39, while the maximum without trimming is more significant than 150. In LastFM dataset, the maximal workload with trimming is 16, while the maximum without trimming is more significant than 100. The CDF plots show that \texttt{Lumos} successfully reduces those devices with large degree values and addresses workload imbalance. 

\begin{figure}[!h]
\centering
\subfloat[Facebook]{\includegraphics[width=1.6in]{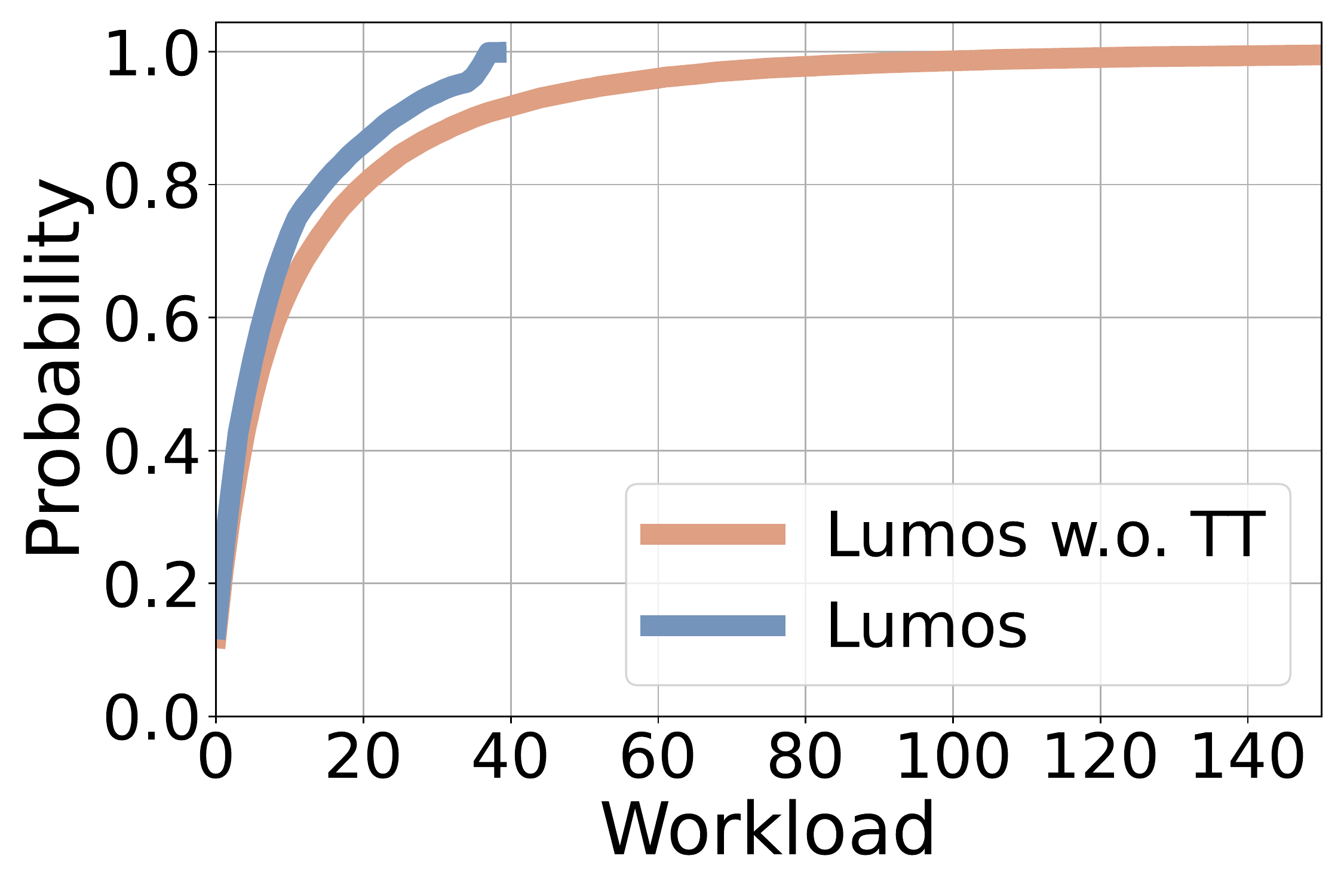}
\label{fig:fb_wl}}
\hfil\hspace{-0.5cm}
\subfloat[LastFM]{\includegraphics[width=1.6in]{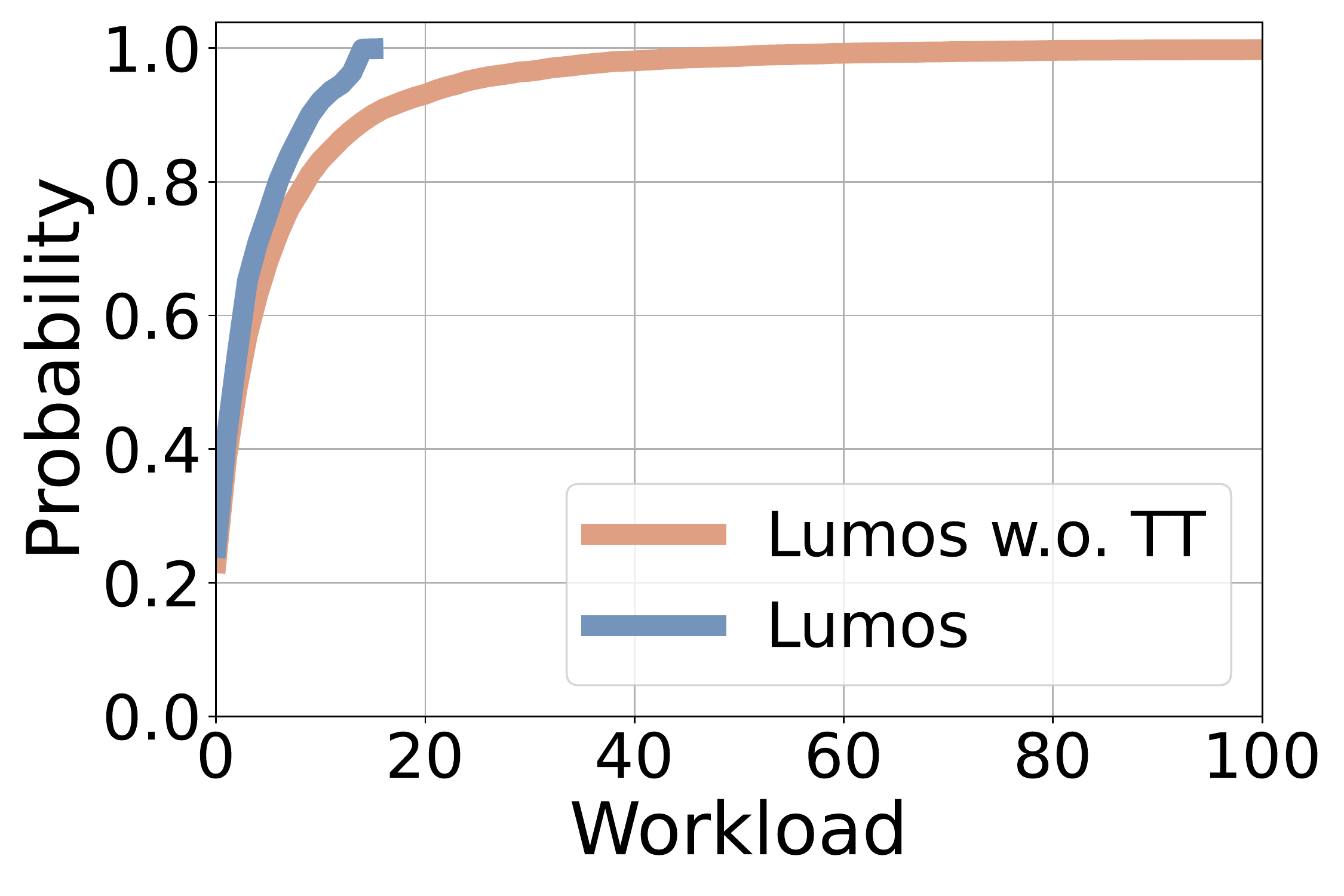}
\label{fig:last_fl}}
\caption{CDF of workload with and without tree trimming(TT)}
\label{fig:fl}
\end{figure}
\subsubsection{Effectiveness in reducing system cost}
 We measure the effect of the tree constructor by quantifying how well the tree constructor improves the system performance. Fig.\ref{fig:perftree} illustrates how tree trimming affects the overall systematic performance of \texttt{Lumos}. Fig.\ref{fig:comm} presents the average number of inter-device communication rounds in one epoch in a tree-based GNN trainer. On Facebook dataset, \texttt{Lumos} with trimming saves 34.19\% communication rounds per local device per epoch for supervised learning while it saves 27.34\% rounds for unsupervised learning. On LastFM dataset, \texttt{Lumos} with trimming saves 43.02\% rounds for supervised learning while it saves 36.80\% rounds for unsupervised learning. Fig.\ref{fig:time} presents the average training time per epoch in tree-based GNN trainer. On Facebook dataset, \texttt{Lumos} with trimming saves 13.32\% training time per epoch for supervised learning while it saves 10.34\% time for unsupervised learning. On LastFM dataset, \texttt{Lumos} with trimming saves 36.38\% training time for supervised learning while it saves 10.90\% for unsupervised learning. Fig.\ref{fig:perftree} shows that our tree constructor with neighbor selection improves the system in terms of communication cost and training time by reducing the workloads in those devices with large degree values. 

\begin{figure}[!h]
\centering
\subfloat[Communication Rounds]{\includegraphics[width=1.6in]{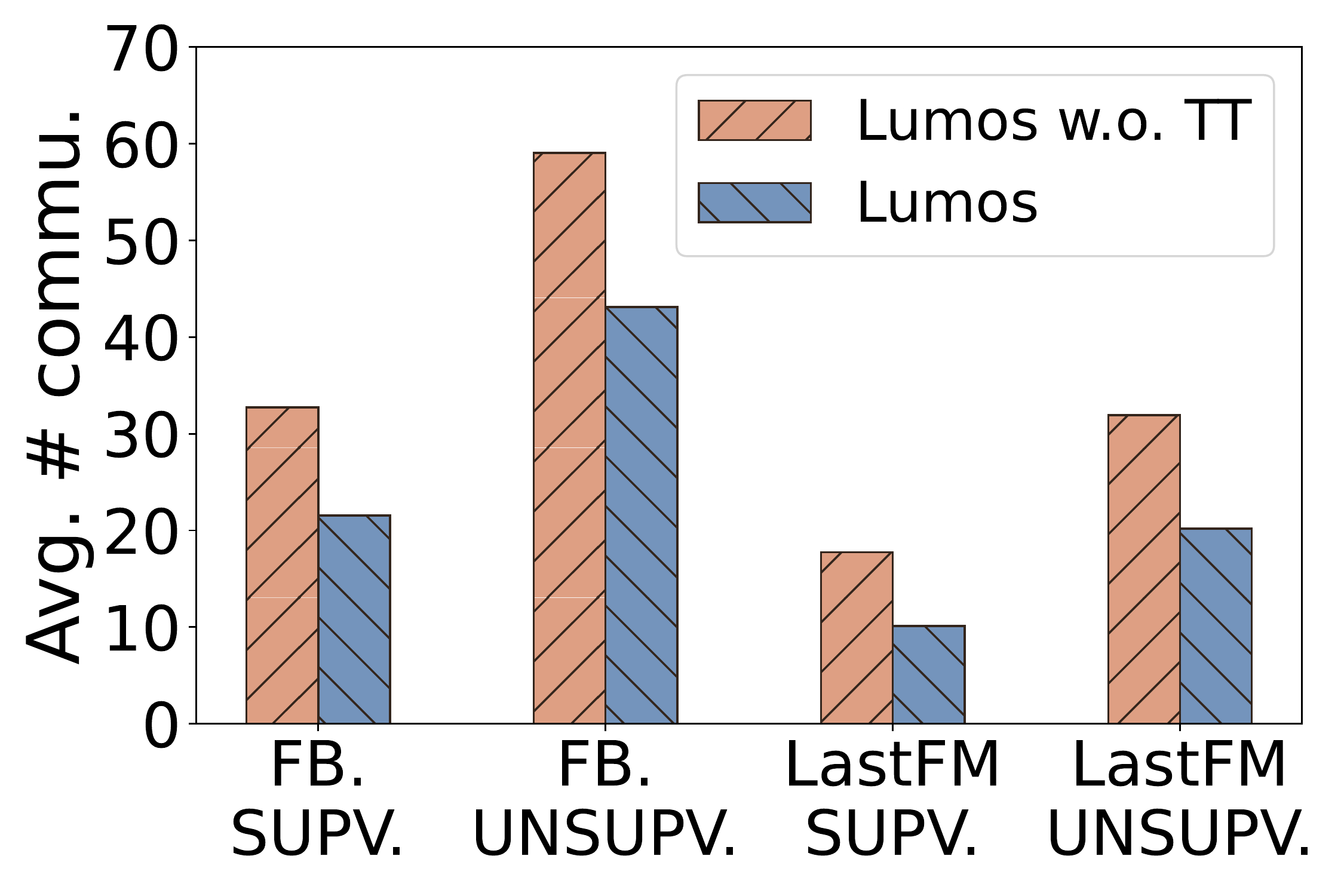}
\label{fig:comm}}
\hfil\hspace{-0.5cm}
\subfloat[Training Time]{\includegraphics[width=1.6in]{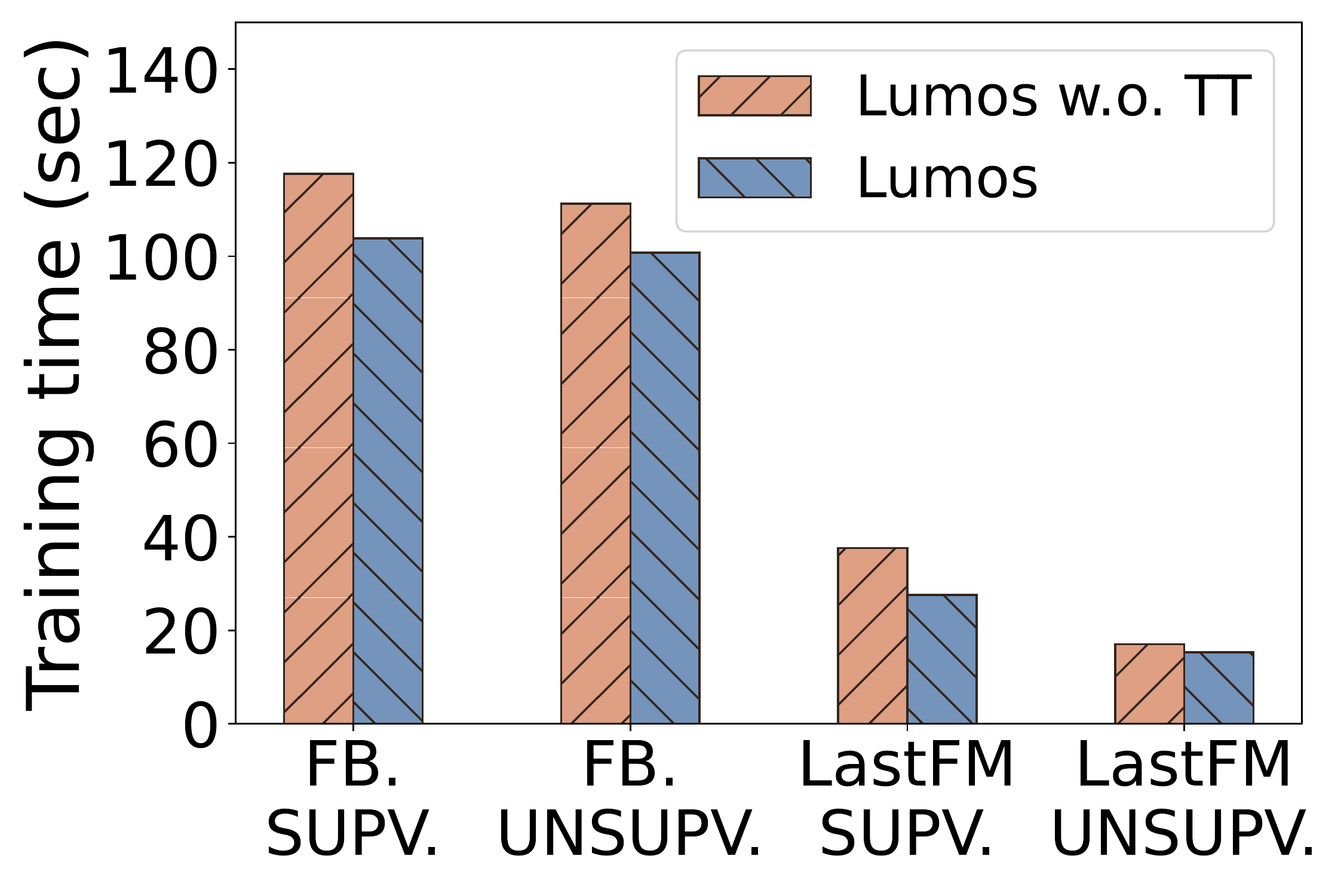}
\label{fig:time}}
\caption{System performance contribution of tree trimming}
\label{fig:perftree}
\end{figure}

 

\section{Conclusion\label{sec:con}}
Federated GNN is an important learning problem that studies efficient graph data processing in privacy preserving environment. Existing studies on federated GNN seldom protect the local node degree. 
We propose the first federated GNN framework, called \texttt{Lumos}, to protect graph statistics in a node-separated setting. We design a novel tree constructor to enhance the expressive power of the original ego network. We propose a MCMC-based algorithm to balance the workload across different devices. We design a tree-based GNN trainer to train the constructed tree with features protected. Experiments have shown that our model outperforms the state-of-art framework in accuracy and system footprints. \texttt{Lumos} satisfies the most natural graph data privacy need in mobile applications and systems, and with a great potential in the emerging decentralized applications. The authors have provided public access to their code at \url{https://github.com/panqy1998/Lumos}. 
\section*{Acknowledgement}
This work is supported in part by SJTU Global Strategic Partnership Fund (2021 SJTU-HKUST). Lingyang Chu's work is supported in part by the NSERC Discovery Grant Program, the CIHR Grant Program, and the McMaster Start-up Grant. We further thank the anonymous reviewers
for their inspiring comments, which help us to improve the
quality of this paper.

\end{document}